\documentclass{article}
\usepackage{pdfpages}

\usepackage[final]{nips_2017}

\usepackage[utf8]{inputenc} % allow utf-8 input
\usepackage[T1]{fontenc}    % use 8-bit T1 fonts
\usepackage{hyperref}       % hyperlinks
\usepackage{url}            % simple URL typesetting
\usepackage{booktabs}       % professional-quality tables
\usepackage{amsfonts}       % blackboard math symbols
\usepackage{nicefrac}       % compact symbols for 1/2, etc.
\usepackage{microtype}      % microtypography

\usepackage{times}
\usepackage{graphicx} % more modern
\usepackage{subcaption}
\usepackage{wrapfig}

\usepackage{natbib}

\usepackage{algorithm}
\usepackage{algorithmic}
\usepackage{mathrsfs}

\usepackage{tikz}
\usetikzlibrary{calc,intersections,through,backgrounds,patterns,angles} % triangle pictures

\usepackage{verbatim}
\usepackage{amsmath, amssymb,amsthm, color}
\usepackage{amsthm} % for proof environment
\usepackage{esvect}
\usepackage{bbm}
\usepackage{dsfont}
\usepackage{tikz}
\usepackage{relsize}
\usepackage{adjustbox}
\usetikzlibrary{arrows,decorations.pathmorphing,fit,positioning}

\def\P{\mathbb P}

\def\pt{\tilde{p}}
\def\Bt{\tilde{B}}
\def\wt{\tilde{w}}

\DeclareMathOperator*{\argmax}{argmax}
\DeclareMathOperator*{\argmin}{argmin}

\DeclareMathOperator*{\conv}{Conv}
\DeclareMathOperator*{\card}{card}
\DeclareMathOperator*{\Dir}{Dir}
\DeclareMathOperator*{\Mult}{Mult}
\newsavebox\myboxA
\newsavebox\myboxB
\newlength\mylenA
\makeatletter
\newcommand*\xoverline[2][0.75]{%
    \sbox{\myboxA}{$\m@th#2$}%
    \setbox\myboxB\null% Phantom box
    \ht\myboxB=\ht\myboxA%
    \dp\myboxB=\dp\myboxA%
    \wd\myboxB=#1\wd\myboxA% Scale phantom
    \sbox\myboxB{$\m@th\overline{\copy\myboxB}$}%  Overlined phantom
    \setlength\mylenA{\the\wd\myboxA}%   calc width diff
    \addtolength\mylenA{-\the\wd\myboxB}%
    \ifdim\wd\myboxB<\wd\myboxA%
       \rlap{\hskip 0.5\mylenA\usebox\myboxB}{\usebox\myboxA}%
    \else
        \hskip -0.5\mylenA\rlap{\usebox\myboxA}{\hskip 0.5\mylenA\usebox\myboxB}%
    \fi}
\makeatother

\title{Conic Scan-and-Cover algorithms for \\ nonparametric topic modeling}

\author{
Mikhail Yurochkin \\
Department of Statistics \\
University of Michigan \\
\texttt{moonfolk@umich.edu}
   \And
Aritra Guha \\
Department of Statistics \\
University of Michigan \\
\texttt{aritra@umich.edu}
   \And
XuanLong Nguyen \\
Department of Statistics \\
University of Michigan \\
\texttt{xuanlong@umich.edu} \\
}

\begin{document}
% \nipsfinalcopy is no longer used
\theoremstyle{definition}
\newtheorem{defn}{Definition}
\newtheorem{thm}{Theorem}
\newtheorem{prop}{Proposition}
\newtheorem{cor}{Corollary}
\newtheorem{problem}{Problem}
\newtheorem{lem}{Lemma}
\newtheoremstyle{TheoremNum}
    {\topsep}{\topsep}              %%% space between body and thm
    {}                      %%% Thm body font
    {}                              %%% Indent amount (empty = no indent)
    {\bfseries}                     %%% Thm head font
    {.}                             %%% Punctuation after thm head
    { }                             %%% Space after thm head
    {\thmname{#1}\thmnote{ \bfseries #3}}%%% Thm head spec
\theoremstyle{TheoremNum}
\newtheorem{lemn}{Lemma}
\newtheorem{thmn}{Theorem}
\newtheorem{propn}{Proposition}

\newcommand{\tikzAngleOfLine}{\tikz@AngleOfLine}
  \def\tikz@AngleOfLine(#1)(#2)#3{%
  \pgfmathanglebetweenpoints{%
    \pgfpointanchor{#1}{center}}{%
    \pgfpointanchor{#2}{center}}
  \pgfmathsetmacro{#3}{\pgfmathresult}%
  }

\maketitle

\begin{abstract}
We propose new algorithms for topic modeling when the number of topics is unknown. 
Our approach relies on an analysis of the concentration of mass and angular geometry of the topic
simplex, a convex polytope constructed by taking the convex hull of vertices representing the latent topics.
Our algorithms are shown in practice to have accuracy comparable to a 
Gibbs sampler in terms of topic estimation, which requires the number of topics be given. 
Moreover, they are one of the fastest among several state of the art parametric techniques.\footnote{Code is available at \url{https://github.com/moonfolk/Geometric-Topic-Modeling}.}
Statistical consistency of our estimator is established under some conditions.
% and additionally propose a procedure for learning number of anchor words in the data, allowing for nonparametric RecoverKL.
\end{abstract}

\section{Introduction}
\label{intro}
A well-known challenge associated with topic modeling inference can be succinctly
summed up by the statement that sampling based approaches may be accurate but computationally very slow, e.g., \cite{pritchard2000inference,griffiths2004finding}, while the variational inference approaches are
faster but their estimates may be inaccurate, e.g., \cite{blei2003latent,hoffman2013stochastic}. 
For nonparametric topic inference, i.e., when the number of topics is a priori unknown, the problem
becomes more acute. 
The Hierarchical Dirichlet Process model \citep{teh2006hierarchical} is an 
elegant Bayesian nonparametric approach which allows for the number of topics to 
grow with data size, but its sampling based inference is much more inefficient
compared to the parametric counterpart.
As pointed out by \cite{yurochkin2016geometric}, the root of the inefficiency can be traced
to the need for approximating the posterior distributions of the latent variables representing the topic labels --- these are not geometrically intrinsic as any permutation of the labels yields the same likelihood. 

%\begin{comment}
%
%Most of the topic modeling inference approaches can be attributed to one of the three lines: probabilistic, matrix factorization and geometric. The former is largely rooted in the formulation of the Latent Dirichlet Allocation and the variational inference techniques \cite{blei2003latent}, later improved using stochastic optimization \citep{hoffman2013stochastic}, and the admixture model with sampling based inference approach by \cite{pritchard2000inference} continued with Collapsed Gibbs sampler \citep{griffiths2004finding}. The above techniques typically assume number of topics to be an input parameter and estimating it remains to be a very challenging problem. An elegant nonparametric solution - Hierarchical Dirichlet Process \citep{teh2006hierarchical} was proposed to address this issue by putting a prior that increases number of topics as the data size grows. Probabilistic approaches rely on estimating the posterior of word-to-topic labels, which is indeed not necessary for learning topics and tends to slow down the performance.
%\end{comment}

A promising approach in addressing the aforementioned challenges is to take a \emph{convex geometric}
perspective, where topic learning and inference may be formulated as a convex geometric 
problem: the observed documents correspond to points randomly drawn from a \emph{topic polytope}, a convex set 
whose vertices represent the topics to be inferred. This perspective has been adopted to establish
posterior contraction behavior of the topic polytope in both theory
and practice \citep{nguyen2015posterior,tang2014understanding}. A method for topic estimation that
exploits convex geometry, the Geometric Dirichlet Means (GDM) algorithm, was proposed by \cite{yurochkin2016geometric}, which demonstrates attractive behaviors both in terms of running 
time and estimation accuracy.
In this paper we shall continue to amplify this viewpoint to address \emph{nonparametric topic modeling},
a setting in which the number of topics is unknown, as is the distribution inside the topic polytope (in 
some situations). %The theoretical result of \cite{nguyen2015posterior} provided a hint that the
%topic simplex can be consistently estimated as long as its vertices are sufficiently separated. 

We will propose algorithms for topic estimation by explicitly accounting for the concentration of mass and angular 
geometry of the topic polytope, typically a simplex in topic modeling applications. 
The geometric intuition is fairly clear: each vertex of the topic 
simplex can be identified by a ray emanating from its center
(to be defined formally), while the concentration of mass can be quantified for the cones hinging
on the apex positioned at the center. Such cones can be rotated around the center to scan for high density regions 
inside the topic simplex --- under mild conditions such cones can be 
constructed efficiently to recover both the number of vertices and their estimates. 

We also mention another fruitful approach, which casts topic estimation as a matrix factorization problem \citep{deerwester1990indexing, xu2003document, anandkumar2012spectral, arora2012practical}.
%The Latent Semantic Analysis \citep{deerwester1990indexing}, one of the first topic modeling approaches, utilizes singular value decomposition, which doesn't allow for the probabilistic interpretation of the topics. Later this was addressed using non-negative matrix factorization (NMF) \citep{xu2003document}.
A notable recent algorithm coming from the matrix factorization perspective is RecoverKL \citep{arora2012practical}, which solves non-negative matrix factorization (NMF) efficiently under assumptions on the existence of
so-called anchor words. RecoverKL remains to be a parametric technique --- we will extend it to a nonparametric setting and show that the 
anchor word assumption appears to limit the number of topics one can efficiently learn.

Our paper is organized as follows. In Section \ref{back} we discuss recent developments in geometric topic modeling and introduce our approach; Sections \ref{cov_th} and \ref{doc_scan} deliver the contributions outlined above; Section \ref{experiment} demonstrates experimental performance; we conclude with a discussion in Section \ref{discussion}.

\section{Geometric topic modeling}
\label{back}
\paragraph{Background and related work}
%\label{g_tm}
In this section we present the convex geometry of the Latent Dirichlet Allocation (LDA) model of \cite{blei2003latent}, 
along with related theoretical and algorithmic results that motivate our work.
%leading to a formulation of several questions that this work addresses. 
Let $V$ be vocabulary size and $\Delta^{V-1}$ be the corresponding vocabulary probability simplex. Sample $K$ topics (i.e., distributions on words) $\beta_k\thicksim\Dir_V(\eta)$, $k=1,\ldots,K$, where $\eta\in\mathbb{R}_{+}^V$. Next, sample $M$ document-word probabilities $p_m$ residing in the \emph{topic simplex} $B:=\conv(\beta_1,\ldots,\beta_K)$ (cf. \cite{nguyen2015posterior}), by first generating their \emph{barycentric coordinates} (i.e., topic proportions) $\theta_m\thicksim\Dir_K(\alpha)$ and then setting $p_m := \sum_k \beta_k \theta_{mk}$ for $m=1,\ldots,M$ and $\alpha \in \mathbb{R}_{+}^K$. Finally, word counts of the $m$-th document can be sampled $w_m\thicksim\Mult(p_m,N_m)$, where $N_m \in \mathbb{N}$ is the number of words in document $m$. The above model is equivalent to the LDA when individual words to topic label assignments are marginalized out. 

\cite{nguyen2015posterior} established posterior contraction rates of the topic simplex, provided that $\alpha_k\leq 1\,\forall k$ and \emph{either} number of topics $K$ is known \emph{or} topics are sufficiently separated in terms of the Euclidean distance. \cite{yurochkin2016geometric} devised an estimate for $B$, taken to be a fixed unknown quantity, by formulating a geometric objective function,
%\begin{equation}
%\label{gdm_obj}
%G(B) = \sum_{m=1}^M N_m \min\limits_{x:x\in B}\|x - \bar w_m\|_2^2,
%\end{equation}
which is minimized when topic simplex $B$ is close to the normalized documents $\bar w_m := w_m/N_m$. They showed that the estimation of topic proportions $\theta_m$ given $B$ simply reduces to taking barycentric coordinates of the projection of $\bar w_m$ onto $B$. 
To estimate $B$ given $K$, they proposed a Geometric Dirichlet Means (GDM) algorithm, which operated
by performing a k-means clustering on the normalized documents, followed by a geometric correction
for the cluster centroids. The resulting algorithm is remarkably fast and accurate, supporting the potential of the geometric approach. The GDM is not applicable when $K$ is unknown, 
but it provides a motivation which our approach is built on.

\paragraph{The Conic Scan-and-Cover approach}
To enable the inference of $B$ when $K$ is not known, we need to investigate
the concentration of mass inside the topic simplex. It suffices to focus on two types of
geometric objects: cones and spheres, which provide the basis for a complete coverage of the simplex.
To gain intuition of our procedure, which we call Conic Scan-and-Cover (CoSAC) approach,
imagine someone standing at a center point of a triangular dark room 
trying to figure out all corners with a portable flashlight, which can produce a \emph{cone} of light.
A  room corner can be identified with the direction of the farthest visible data objects. Once a corner is found, 
one can turn the flashlight to another direction to scan for the next ones. See Fig. \ref{fig:p_remain}, 
where red denotes the scanned area. To make sure that all corners 
are detected, the cones of light have to be open to an appropriate range of angles so that enough data objects
can be captured and removed from the room. To make sure no false corners are declared,
we also need a suitable stopping criterion, by relying only on data points that lie beyond a certain spherical
radius, see Fig. \ref{fig:add_r}. Hence, we need to be able to gauge the concentration of mass for 
suitable cones and spherical balls in $\Delta^{V-1}$.  This is the subject of the next section.

\begin{figure*}{t!}
\vskip -0.1in
\begin{subfigure}{.33\textwidth}
  \centering
  \captionsetup{justification=centering}
  \includegraphics[width=\linewidth]{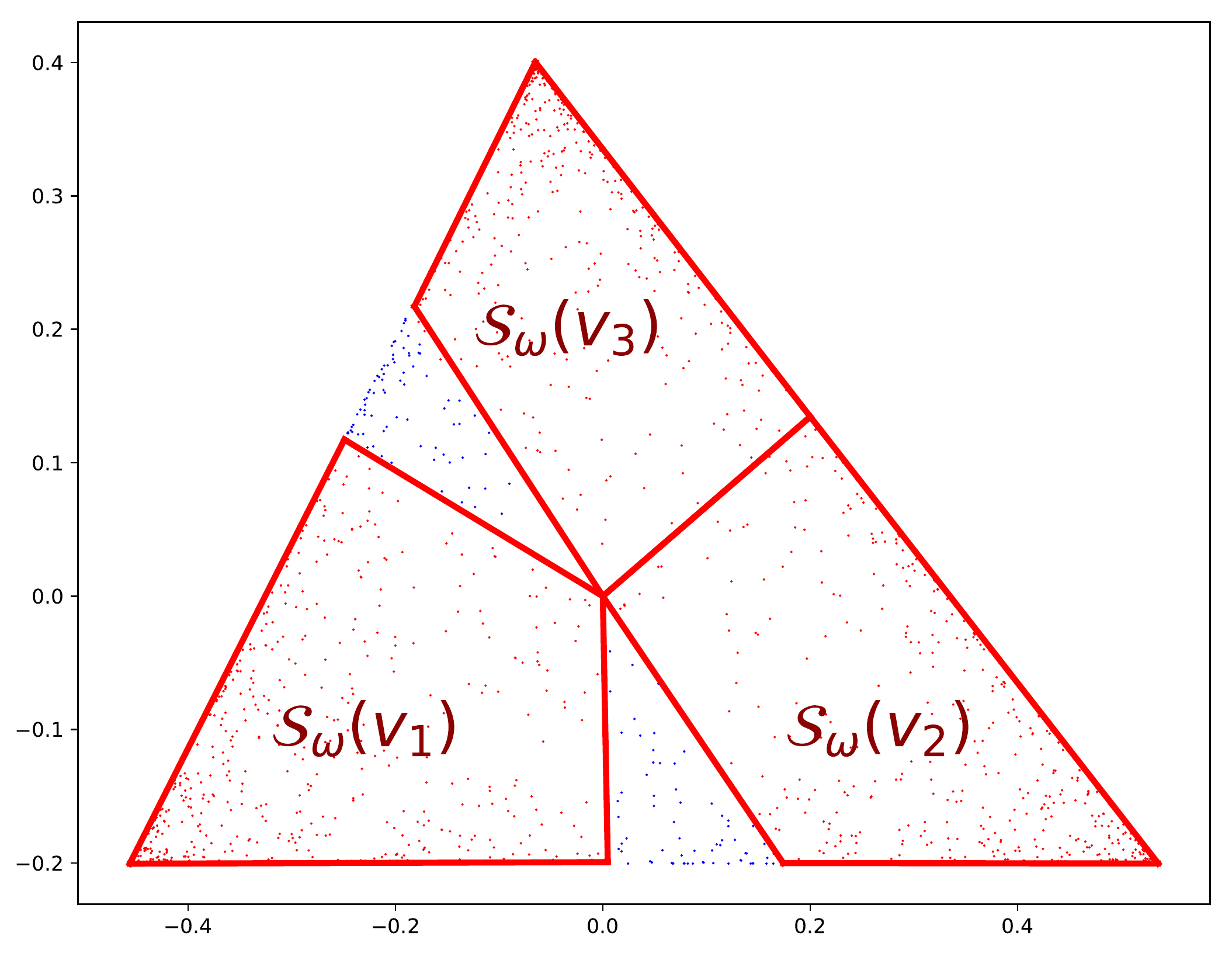}
  \caption{An incomplete coverage using \\ 3 cones (containing red points).}
  \label{fig:p_remain}
\end{subfigure}
\begin{subfigure}{.33\textwidth}
  \centering
  \captionsetup{justification=centering}
  \includegraphics[width=\linewidth]{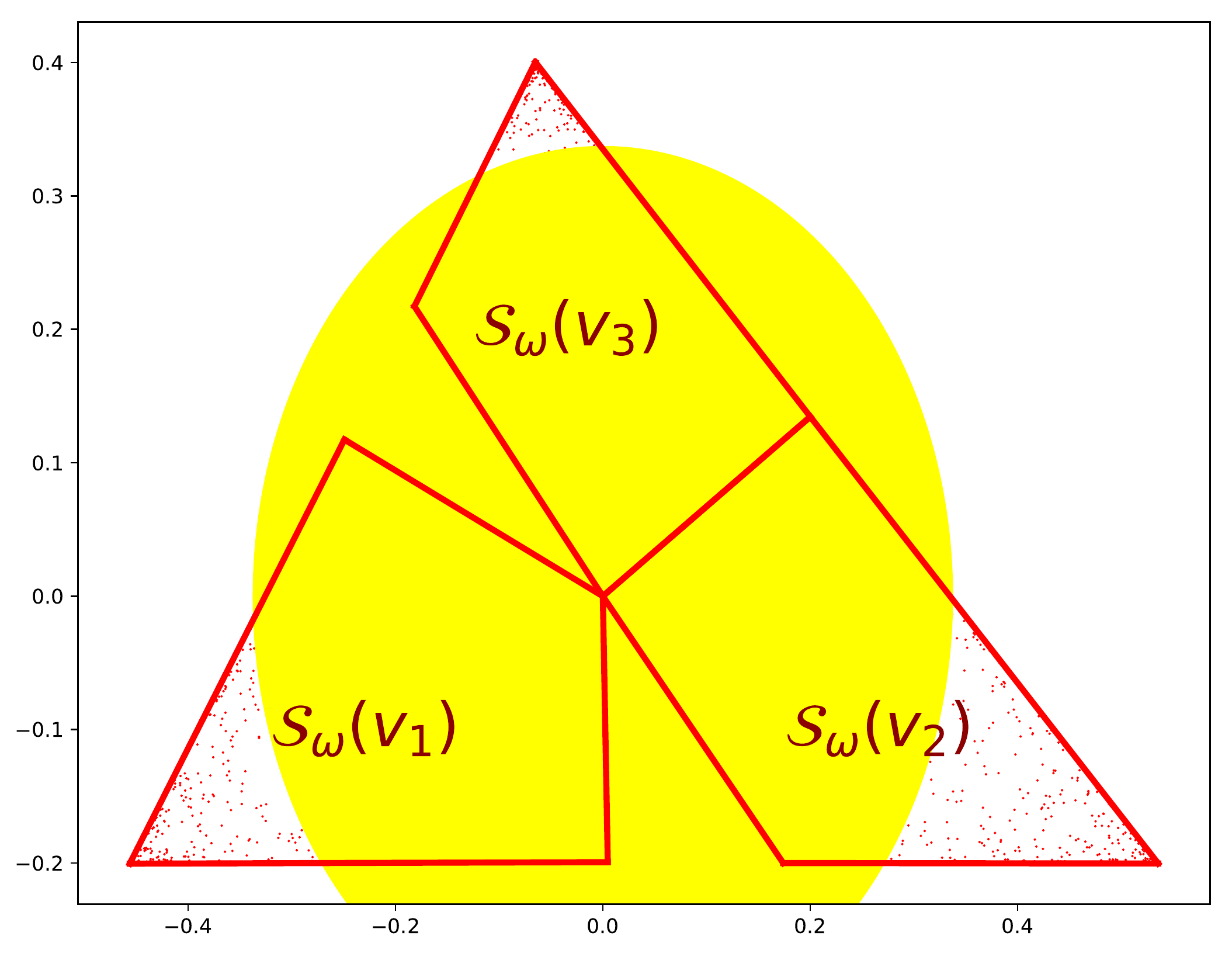}
  \caption{Complete coverage using \\ 3 cones (red) and a ball (yellow).}
  \label{fig:add_r}
\end{subfigure}
\begin{subfigure}{.33\textwidth}
  \centering
  \includegraphics[width=\linewidth]{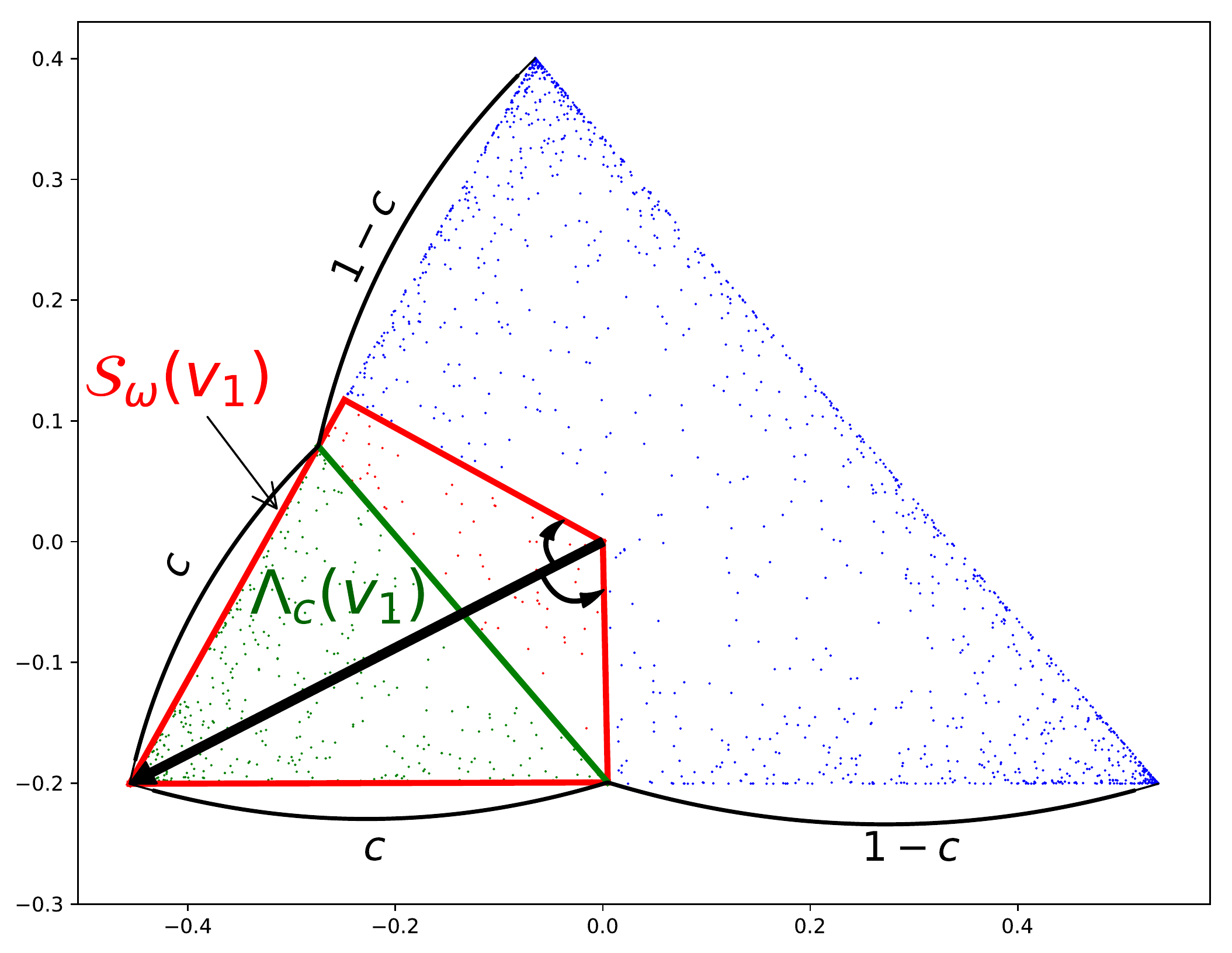}
  \caption{Cap $\Lambda_c(v_1)$ and cone $\mathcal{S}_\omega(v_1).$}
  \label{fig:n_cone}
\end{subfigure}%
\caption{Complete coverage of topic simplex by cones and a spherical ball for $K=3$, $V=3$.}
\label{fig:geom}
\vskip -0.2in
\end{figure*}

\section{Geometric estimation of the topic simplex}
\label{cov_th}
We start by representing $B$ in terms of its convex and angular geometry. First, $B$ is centered at a point
denoted by $C_p$. The centered probability simplex is denoted by
%$C_p:= \E p_1 \in \Delta^{V-1}$ --- the population mean of the data. 
$\Delta^{V-1}_0 := \{x\in\mathbb{R}^V | x + C_p \in \Delta^{V-1}\}$. Then,
write $b_k := \beta_k - C_p \in \Delta^{V-1}_0$ for $k=1,\ldots,K$ and $\pt_m := p_m - C_p \in \Delta^{V-1}_0$ for $m=1,\ldots,M$. Note that re-centering leaves corresponding barycentric coordinates $\theta_m\in\Delta^{K-1}$ unchanged.
Moreover, the extreme points of centered topic simplex $\Bt := \conv\{b_1,\ldots,b_K\}$ can now be represented by 
their directions $v_k\in\mathbb{R}^V$ and corresponding radii $R_k \in \mathbb{R}_+$ such that 
$b_k = R_k v_k$ for any $k = 1,\ldots, K$.

%$B$ can now be represented by their directions $v_k\in\mathbb{R}^V$ and radii $R_k \in \mathbb{R}_+$ such that $b_k = R_k v_k$ for any $k$.
%We will start by analyzing noiseless case, i.e. when we have access to true document generating distributions $p_1,\ldots,p_M$ and proceed with the analysis of classic document scenario, where only noisy estimates $\bar w_m = w_m/N_m$ (i.e. normalized documents) of corresponding $p_m$-s are available.

\subsection{Coverage of the topic simplex}

 The first step toward formulating a CoSAC approach is to show how 
 $\Bt$ can be \emph{covered} with exactly $K$ cones and one spherical ball positioned at $C_p$. 
 A cone is defined as set $\mathcal{S}_\omega(v) := \{ p \in \Delta^{V-1}_0 |  d_{\cos}(v,p)<\omega \}$, where 
 we employ the angular distance (a.k.a. cosine distance) $d_{\cos}(v,p) := 1 - \cos(v,p)$ and $\cos(v,p)$ is the cosine of angle $\angle(v,p)$ formed by vectors $v$ and $p$. 
 %We will show that $\Bt$ can be covered by $K$ suitable cones, next clarifying the role of the sphere and analyzing distributional properties of the cone.

\paragraph{The Conical coverage} It is possible to choose $\omega$ so that the topic simplex can be covered with exactly $K$ cones, that is, $\bigcup\limits_{k=1}^K \mathcal{S}_{\omega}(v_k) \supseteq \Bt$. Moreover, each cone contains exactly one vertex. Suppose that $C_p$ is the \emph{incenter} of the topic simplex $\Bt$, with $r$ being the inradius. The incenter and inradius correspond to the maximum volume sphere contained in $\Bt$. Let $a_{i,k}$ denote the distance between the $i$-th and $k$-th vertex of $\Bt$, with  $a_{min} \leq a_{i,k} \leq a_{max}$ for all $i,k$, and $R_{max},R_{min}$ such that $R_{min}\leq R_k := \|b_k\|_2 \leq R_{max} \ \forall \ k=1,\ldots,K$. Then we can establish
the following.
\begin{prop}
\label{delta_exist}
For simplex $\Bt$ and
$\omega \in (\omega_1,\omega_2)$, where $\omega_1 = 1- r/ R_{max}$ 
and $\omega_2 = \max \{ (a_{max}^2)/(2R_{max}^2),\max\limits_{i,k=1,\ldots,K}(1-\cos(b_i,b_k)\}$,
the cone $\mathcal{S}_\omega(v)$ around any vertex direction $v$ of $\Bt$ contains exactly one vertex. 
Moreover, complete coverage holds: 
$\bigcup\limits_{k=1}^K \mathcal{S}_{\omega}(v_k) \supseteq \Bt$.
\end{prop}

We say there is an \emph{angular separation} if $\cos(b_i,b_k) \leq 0$ for any $i,k=1,\ldots,K$ (i.e., the
angles for all pairs are at least $\pi/2$), then $\omega \in \left(1 - \frac{r}{R_{max}}, 1 \right)\neq \emptyset$. 
Thus, under angular separation, the range $\omega$ that allows for full coverage is nonempty independently of $K$. Our result is in agreement with that of \cite{nguyen2015posterior}, whose result suggested that topic simplex $B$ can be consistently estimated without knowing $K$, provided there is a minimum edge length $a_{min} >0$.
%The advantage of the angular geometric viewpoint is that it will allow us to define first nonparametric geometric topic modeling algorithm. 
The notion of angular separation leads naturally to the Conic Scan-and-Cover algorithm. Before getting there,
we show a series of results allowing us to further extend the range of admissible $\omega$.
%\mikcomm{explain condition in terms of angular separation} Indeed for an equilateral $\Bt$, the condition of the Theorem \ref{delta_exist} is satisfied and $\omega \in (1-1/\sqrt{K-1},1+1/K)$ will ensure the desired coverage.
%Nonetheless we would want to further extend range of admissible $\omega$. Fig. \ref{fig:p_remain} shows uncompleted coverage of $\Bt$ with $K=3$ vertices and $\omega=0.3$. 

The inclusion of a spherical ball centered at $C_p$ allows us to expand substantially the
range of $\omega$ for which conical coverage continues to hold. In particular, we can 
reduce the lower bound on $\omega$ in Proposition \ref{delta_exist}, since we only need to cover the regions 
near the vertices of $\Bt$ with cones using the following proposition.  Fig. \ref{fig:add_r} provides an illustration.
% might be not as robust for non-equilateral $\Bt$ and with finite and noisy data, therefore a smaller $\omega$ is preferred, but then set $A$ will be non empty after $K$ iterations as in Fig. \ref{fig:p_remain}, where the blue regions correspond to $A$ after $K=3$ iterations of the scan. It appears that a particular property of the the remainder points in $A$ is \emph{circumscribing} in both regular and geometric sense - they all have a small norm, i.e. $\exists\ r:\ \|\pt_i\|_2<r\ \forall\ i\in A$, therefore we modify the stopping criteria of the scan to occur when all remaining points in $A$ have norm smaller than a given parameter $r$, which is shown in Fig. \ref{fig:add_r}. Observe that now $\omega$ can be reduced significantly as we only need to be able to scan through the corners of $\Bt$.
%
%Suppose we have the knowledge on $\mathcal{R}$, the radius of the sphere mentioned above. 
\begin{prop}
\label{r_delta}
Let $\mathscr{B}(C_p,\mathcal{R})=\{ \tilde{p} \in \mathbb{R}^{V} |  \|\tilde{p} -C_p\|_2 \leq \mathcal{R} \}$, $\mathcal{R} > 0$; $\omega_1,\omega_2$ given in Prop. \ref{delta_exist}, and
\begin{equation}
\label{w3}
\omega_3 := 1- \min\bigg\{\underset{i,k}\min\left( \frac{R_k \sin^2(b_i,b_k)}{\mathcal{R}} + \cos(b_i,b_k)\sqrt{1-\frac{R_k^2 \sin^2(b_i,b_j)}{\mathcal{R}^2}}\right),1\bigg\},
\end{equation}
then we have 
$\bigcup\limits_{k=1}^K \mathcal{S}_{\omega}(v_k) \cup {\mathscr{B}(C_p,\mathcal{R})} \supseteq \Bt$ whenever $\omega \in (\min\{\omega_1, \omega_3\},\omega_2)$.
\end{prop}
Notice that as $\mathcal{R} \rightarrow R_{max}$ , the value of 
$\omega_3 \rightarrow 0$. Hence if $\mathcal{R} \leq R_{min} \approx R_{max}$, the
admissible range for $\omega$ in Prop.  \ref{r_delta} results in a substantial strengthening from Prop. \ref{delta_exist}. It is worth noting that the above two geometric propositions do not require any distributional properties inside the simplex.

\paragraph{Coverage leftovers} In practice complete coverage may fail if
$\omega$ and $\mathcal{R}$ are chosen outside of corresponding ranges suggested by the previous two propositions.
In that case, it is useful to note that leftover regions will have a very low mass. 
Next we quantify the mass inside a cone that \emph{does} contain a vertex, which allows us to \emph{reject} a cone that has low mass, therefore not containing a vertex in it.
\begin{prop}
\label{n_eq} The cone $S_\omega(v_1)$ whose axis is a topic direction $v_1$ has mass
\begin{eqnarray}
\left.\begin{aligned}
& \P(\mathcal{S}_\omega(v_1)) > \P(\Lambda_{c}(b_1)) = \frac{\int_{1-c}^1\theta_1^{\alpha_1-1}(1- \theta_1)^{\sum_{i\neq 1}\alpha_i -1}\mathrm{d}\theta_1}{\int_{0}^1\theta_1^{\alpha_1-1}(1- \theta_1)^{\sum_{i\neq 1}\alpha_i -1}\mathrm{d}\theta_1} = \\
& \frac{c^{\sum_{i \neq 1}\alpha_i}(1-c)^{\alpha_1}\Gamma(\sum_{i=1}^K \alpha_i)}{(\sum_{i \neq 1} \alpha_i) \Gamma(\alpha_1) \Gamma(\sum_{i \neq 1} \alpha_i)}
\biggr [1 + \frac{c\sum_{i=1}^K \alpha_i }{\sum_{i \neq 1} \alpha_i + 1}  + \frac{c^2(\sum_{i=1}^K \alpha_i)(\sum_{i=1}^K \alpha_i +1) }{(\sum_{i \neq 1} \alpha_i + 1)(\sum_{i \neq 1} \alpha_i + 2)} + \cdots  
\biggr ],
\end{aligned}\right.
\end{eqnarray}
where $\Lambda_{c}(b_1)$ is the simplicial cap of $\mathcal{S}_\omega(v_1)$  which is composed
of vertex $b_1$ and a base parallel to the corresponding base of $\Bt$ and cutting adjacent edges of $\Bt$ in the ratio $c:(1-c)$. 
\end{prop}
See Fig. \ref{fig:n_cone} for an illustration for the simplicial cap described in the proposition.
Given the lower bound for the mass around a cone containing a vertex,
we have arrived at the following guarantee.
\begin{prop}
\label{n_theorem} For $\lambda\in (0,1)$, let  $c_\lambda$ be such that $\lambda =\underset{k} \min \; \P(\Lambda_{c_{\lambda}}(b_k))$ and let $\omega_\lambda$ be such that
\begin{equation}
\label{c_lambda}
c_\lambda=\left(\left(2\sqrt{1- \frac{r^2}{R_{max}^2}}\right) \  \left( \sin(d)\cot(\arccos(1-\omega_\lambda)) + \cos(d)  \right) \right)^{-1},
\end{equation}
where angle $d \leq \underset{i,k}{\operatorname{min}}\, \angle(b_k, b_k - b_i)$. Then, as long as
\begin{equation}
\label{omega_lambda}
\omega\in\left(\omega_\lambda, \max\left(\frac{a_{max}^2}{2R_{max}^2},\max\limits_{i,k=1,\ldots,K}(1-\cos(b_i,b_k)\right)\right),
\end{equation} 
the bound $\P(\mathcal{S}_{\omega}(v_k)) \geq \lambda\text{ holds for all }k=1,\ldots, K$.
\end{prop}
%\begin{proof}
%Consider Fig. \ref{fig:R_delta}, with length of $AC = a_{i,k}c$, where $c$ is the proportion in which the cone cuts $AC$, the edge joining vertex $i$ and vertex $k$. Now, from sine laws of a triangle.
%\begin{eqnarray}\label{c_n dependence}
%\frac{R_k}{a_{i,k}c}=\sin(d_{i,k}) \cot{\phi_{i,k}} + \cos(d_{i,k})
%\end{eqnarray}
%where, $\phi_{i,k}$ is as defined in the proof of Theorem \ref{r_delta}. Now $\frac{a_{i,k}}{R_k}\leq \frac{2(\sqrt{R_{max}^2-r^2})}{R_{max}}$. The  choice of $\phi_n= \cos{\omega_n}$ satisfying
%\begin{eqnarray}
%c_n \geq \frac{1}{2\sqrt{1-\frac{r^2}{R_{max}^2}}} \underset{i,k}{\operatorname{min}}\frac{1}{\sin(d_{i,k}) \cot{\phi_n} + \cos(d_{i,k})}
%\end{eqnarray}
%therefore proves the theorem. Since,$\phi_n\leq \frac{\pi}{2} - d_{i,k}$, for all $i,k$,  the function $\sin(d_{i,k}) \cot{\phi_n} + \cos(d_{i,k})$ is increasing in $d_{i,k}$, as can be checked  for maxima by first derivative rule. USing cosine rule,
%\begin{eqnarray}
%d_{i,k}=\frac{-R_i^2 +R_k^2 + a_{i,k}^2}{2a_{i,k}R_k}.
%\end{eqnarray}
%Maximizing this quantity with rspect to $i$ and $k$ we get the result.
%\end{proof}
%Proofs of all propositions can be found in the Supplement.
%\mikcomm{Approximation might go to algorithm parameter selection part} Assuming small $\alpha_k=1/K,\ \forall k$ we can obtain a simple approximation 
%$\P(\Lambda_{c_1}) \approx \frac{c_1^{(K-1)/K}}{(K-1)(1-c_1)}$.
%\mikcomm{If you have 3 parameter theorem add it here}

\subsection{CoSAC: Conic Scan-and-Cover algorithm}
\label{cone_s}
Having laid out the geometric foundations, we are ready
to present the Conic Scan-and-Cover (CoSAC) algorithm, which is a scanning procedure for detecting the
presence of simplicial vertices based on data drawn randomly from the simplex. The idea is simple:
iteratively pick the farthest point from the center estimate $\hat C_p := \frac{1}{M}\sum_m p_m $,
say $v$, then construct a cone $\mathcal{S}_\omega(v)$ for some suitably chosen $\omega$, and
remove all the data residing in this cone. Repeat until there is no data point left.

Specifically,
let $A = \{1,\ldots,M\}$ be the index set of the initially unseen data, then set
$v := \argmax\limits_{\pt_m:m\in A}\|\pt_m\|_2$ and update $A := A\setminus \mathcal{S}_\omega(v)$.
The parameter $\omega$ needs to be sufficiently large to ensure that the farthest point is a good estimate
of a true vertex, and that the scan will be completed in exactly $K$ iterations; $\omega$ needs to be
not too large, so that  $\mathcal{S}_\omega(v)$ does not contain more than one vertex. 
The existence of such $\omega$ is guaranteed by Prop. \ref{delta_exist}. 
In particular, for an equilateral $\Bt$, the condition of the Prop. \ref{delta_exist} is satisfied as long as $\omega \in (1-1/\sqrt{K-1},1+1/(K-1))$. 

In our setting, $K$ is unknown. A smaller $\omega$ would be a more robust choice, and accordingly the set 
$A$ will likely remain non-empty after $K$ iterations. See the illustration of Fig. \ref{fig:p_remain}, where the blue regions correspond to $A$ after $K=3$ iterations of the scan. As a result, we proceed by adopting a stopping criteria based on Prop. \ref{r_delta}:
the procedure is stopped as soon as $\forall \ m\in A \ \|\pt_m\|_2<\mathcal{R}$, which allows us to complete the scan in $K$ iterations (as in Fig. \ref{fig:add_r} for $K=3$).
%This allows us to complete the scan in $K=3$ iterations as in Fig. \ref{fig:add_r}.

%Suppose a hyperplane parallel to the base cuts the sides adjacent to the apex in the ratio $c:(1-c)$ , with $c$ being the proportional part closer to the apex. Then with $\alpha_i= 1/K $ for all $i$ in Proposition \ref{n_eq}, we get  for $c <<< 1$, 
%$ \frac{c^{1-1/K}(1-c)^{1/K}}{K-1} \leq \P(\Lambda_{c}) \leq \frac{1}{K-1}(\frac{c}{1-c})^{1-1/K}$. Choosing $\mathcal{R}$ as median length is equivalent to choosing $c$ such that $1-\frac{K}{K-1}(\frac{c}{1-c})^{1-1/K} \geq 1/2$, leading to choice of $c\leq (\frac{K-1}{2K})^{K/(K-1)}$.For this choice of $c$, we see from the calulations on the lower bound of $\P(\Lambda_{c}) $, that for $K \leq 10^4/5$, $\P(\Lambda_{c}) \geq 10^{-4} $, thereby showing that if we construct $\mathcal{R}$ as the median of all radial distances from $C_p$, then the resulting cones around any topic contains sufficient mass, so as not to eliminate the topic .\\ 

The CoSAC algorithm is formally presented by Algorithm \ref{angle_simple}. Its running is illustrated in Fig. \ref{fig:algo_iter}, where we show iterations 1, 26, 29, 30 of the algorithm by plotting norms of the centered documents in the active set $A$ and cone $\mathcal{S}_\omega(v)$ against cosine distance to the chosen direction of a topic. Iteration 30 (right) satisfies stopping criteria and therefore CoSAC recovered correct $K=30$. Note that this type of visual representation can be useful in practice to verify choices of $\omega$ and $\mathcal{R}$.
The following theorem establishes the consistency of the CoSAC procedure.
\begin{thm}
\label{consistency}
Suppose $\{\beta_1,\dots,\beta_K\}$ are the true topics, incenter $C_p$ is
given, $\theta_m\thicksim\Dir_K(\alpha)$ and $p_m := \sum_k \beta_k \theta_{mk}$ for $m=1,\ldots,M$ and $\alpha \in \mathbb{R}_{+}^K$. Let $\hat{K}$ be the estimated number of topics,
$\{\hat \beta_1,\dots,\hat \beta_{\hat{K}}\}$ be the output of Algorithm \ref{angle_simple} trained with $\omega$ and $\mathcal{R}$ as in Prop. \ref{r_delta}. Then $\forall \ \epsilon >0$,
\begin{eqnarray*}
\P \left ( \left \{\underset{j \in \{1, \dots, \hat{K}\}}\min\|\beta_i - \hat \beta_j \| > \epsilon \ \text{, for any $i$}\in \ \{1,\dots,\hat{K}\} \right \} \cup\{ K \neq \hat{K}\} \right ) \rightarrow 0\text{ as }M\rightarrow \infty.
\end{eqnarray*}
\end{thm}

\paragraph{Remark} 
We found the choices $\omega=0.6$ and $\mathcal{R}$ to be median of $\{\|\pt_1\|_2,\ldots,\|\pt_M\|_2\}$ to be robust in practice and agreeing with our theoretical results. From Prop. \ref{n_eq} it follows that choosing $\mathcal{R}$ as median length is equivalent to choosing $\omega$ resulting in an edge cut ratio $c$ such that $1-\frac{K}{K-1}(\frac{c}{1-c})^{1-1/K} \geq 1/2$, then  $c\leq (\frac{K-1}{2K})^{K/(K-1)}$, which, for any equilateral topic simplex $B$, is satisfied by setting $\omega\in(0.3,1)$,
provided that $K\leq 2000$ based on the Eq. \eqref{c_lambda}.

\section{Document Conic Scan-and-Cover algorithm}
\label{doc_scan}
In the topic modeling problem, $p_m$ for $m=1,\ldots, M$ are \emph{not} given. Instead,
under the bag-of-words assumption, we are given the frequencies of words in documents $w_1,\ldots,w_M$ 
which provide a point estimate $\bar w_m := w_m/N_m$ for the $p_m$. Clearly, if number of documents $M\rightarrow \infty$ and length of documents $N_m\rightarrow \infty\ \forall m$, we can use 
Algorithm \ref{angle_simple} with the plug-in estimates $\bar w_m$ in place of $p_m$, since 
$\bar w_m \rightarrow p_m$. Moreover, $C_p$ will be estimated by 
$\hat C_p := \frac{1}{M}\sum \bar w_m$.
In practice, $M$ and $N_m$ are finite, some of which may take relatively small values.
Taking the topic direction to be the farthest point in the topic simplex, i.e.,  
$v = \argmax\limits_{\wt_m:m\in A}\|\wt_m\|_2$, where $\wt_m:=\bar w_m - \hat C_p \in \Delta^{V-1}_0$,
may no longer yield a robust estimate, because the variance of this topic direction estimator can be quite high
(in Proposition \ref{supp:var} we show that it is upper bounded with $(1-1/V)/N_m$).

To obtain improved estimates, we propose a technique that we call ``mean-shifting''. 
Instead of taking the farthest point in the simplex, this 
technique is designed to shift the estimate of a topic to a high density region, where true topics are likely to be found.
Precisely, given a (current) cone $\mathcal{S}_\omega(v)$, we re-position
the cone by updating $v := \argmin\limits_{v} \sum_{m\in \mathcal{S}_\omega(v)}\|\wt_m\|_2(1-\cos(\wt_m,v))$.
In other words, we re-position the cone by centering it around the \emph{mean direction} of the cone weighted by 
the norms of the data points inside, which is simply given by $v \propto \sum_{m \in \mathcal{S}_\omega(v)}\wt_m/\card(\mathcal{S}_\omega(v))$. 
This results in reduced variance of the topic direction estimate, due to the averaging over 
data residing in the cone.  

The mean-shifting technique may be slightly modified and taken as a local update for 
a subsequent optimization which cycles through
the entire set of documents and iteratively updates the cones. The optimization is with respect
to the following weighted spherical k-means objective:
\begin{equation}
\label{sph_means}
\min\limits_{\|v_k\|_2=1, k=1,\ldots K}\sum_{k=1}^K\sum_{m\in S^k(v_k)} \|\wt_m\|_2(1-\cos(v_k, \wt_m)),
\end{equation}
where cones $S^k(v_k) = \{m | d_{\cos}(v_k,\pt_m)<d_{\cos}(v_l,\pt_i)\ \forall l \neq k \}$ 
yield a disjoint data partition $\bigsqcup\limits_{k=1}^K S^k(v_k) = \{1,\ldots,M\}$ (this is different from $\mathcal{S}_\omega(v_k)$).
The rationale of spherical k-means optimization is to use full data for estimation of topic directions, 
hence further reducing the variance due to short documents. The
connection between objective function \eqref{sph_means} and topic simplex estimation is given in the \hyperref[supp:sph]{Supplement}. 
Finally, obtain topic norms $R_k$ along the directions $v_k$ using maximum projection: $R_k := \max\limits_{m:m \in S^k(v_k)} \langle v_k, \wt_m \rangle$. Our entire procedure is summarized in Algorithm \ref{angle_full}.

\paragraph{Remark}
In Step 9 of the algorithm, cone $\mathcal{S}_\omega(v)$ with a very low cardinality, i.e., $\card(\mathcal{S}_\omega(v))< \lambda M$, for some small constant $\lambda$, is discarded because this is likely an outlier region that does not actually contain a true vertex. The choice of $\lambda$ is governed by results of Prop. \ref{n_theorem}. For small $\alpha_k=1/K,\ \forall k$, $\lambda \leq \P(\Lambda_c) \approx \frac{c^{(K-1)/K}}{(K-1)(1-c)}$ and for an equilateral $\Bt$ we can choose $d$  such that $\cos(d)=\sqrt{\frac{K+1}{2K}}$. Plugging these values into Eq. \eqref{c_lambda} leads to  $c=\left(\left(2\sqrt{1- \frac{1}{K^2}}\right) \  \left( \sqrt{\frac{K-1}{2K}}(\frac{1-\omega}{\sqrt{1-(1-\omega)^2}}) + \sqrt{\frac{K+1}{2K}}  \right) \right)^{-1}.$ Now, plugging in $\omega=0.6$ we obtain $\lambda\leq K^{-1}$ for large $K$. Our approximations were based on large $K$ to get a sense of $\lambda$, we now make a conservative choice $\lambda=0.001$, so that $(K)^{-1}>\lambda\ \forall K<1000$. As a result, 
a topic is rejected if the corresponding cone contains less than 0.1\% of the data.

\paragraph{Finding anchor words using Conic Scan-and-Cover}
\label{a_rec}
Another approach to reduce the noise is to consider the problem from a different viewpoint, where Algorithm \ref{angle_simple} will prove itself useful. RecoverKL by \cite{arora2012practical} can identify topics with diminishing
errors (in number of documents $M$), \emph{provided} that topics contain anchor words. The problem of finding anchor words geometrically reduces to identifying rows of the word-to-word co-occurrence matrix that form a simplex containing other rows of the same matrix (cf. \cite{arora2012practical} for details). An advantage of this approach is that noise in the word-to-word co-occurrence matrix goes to zero as $M\rightarrow \infty$ no matter the document lengths, hence we can use Algorithm \ref{angle_simple} with "documents" being rows of the word-to-word co-occurrence matrix to learn anchor words nonparametrically and then run RecoverKL to obtain topic estimates. We will call this procedure cscRecoverKL.

\begin{algorithm}[ht]
\caption{Conic Scan-and-Cover (CoSAC)}
\label{angle_simple}
\begin{algorithmic}[1]
\REQUIRE document generating distributions $p_1,\ldots,p_M$, \\
angle threshold $\omega$, norm threshold $\mathcal{R}$
\ENSURE topics $\beta_1,\ldots,\beta_k$
\STATE $\hat C_p = \frac{1}{M}\sum_m p_m $ \COMMENT{find center};\qquad $\pt_m := p_m - \hat C_p$ for $m=1,\ldots,M$ \COMMENT{center the data}
\STATE $A_1 =\{1,\ldots,M\}$ \COMMENT{initialize active set};\qquad $k=1$ \COMMENT{initialize topic count}
\WHILE{$\exists m \in A_k: \|\pt_m\|_2 > \mathcal{R}$}
\STATE $v_k = \argmax\limits_{\pt_m:m \in A_k}\|\pt_m\|_2$ \COMMENT{find topic}
\STATE $\mathcal{S}_\omega(v_k) = \{m:d_{\cos}(\pt_m,v_k)<\omega\}$ \COMMENT{find cone of near documents}
\STATE $A_k = A_k \setminus \mathcal{S}_\omega(v_k)$ \COMMENT{update active set}
\STATE $\beta_k = v_k + \hat C_p$, $k=k+1$ \COMMENT{compute topic}
\ENDWHILE
\end{algorithmic}
\end{algorithm}

\begin{figure*}[ht]
\vskip -0.1in
\begin{center}
\centerline{\includegraphics[width=\textwidth]{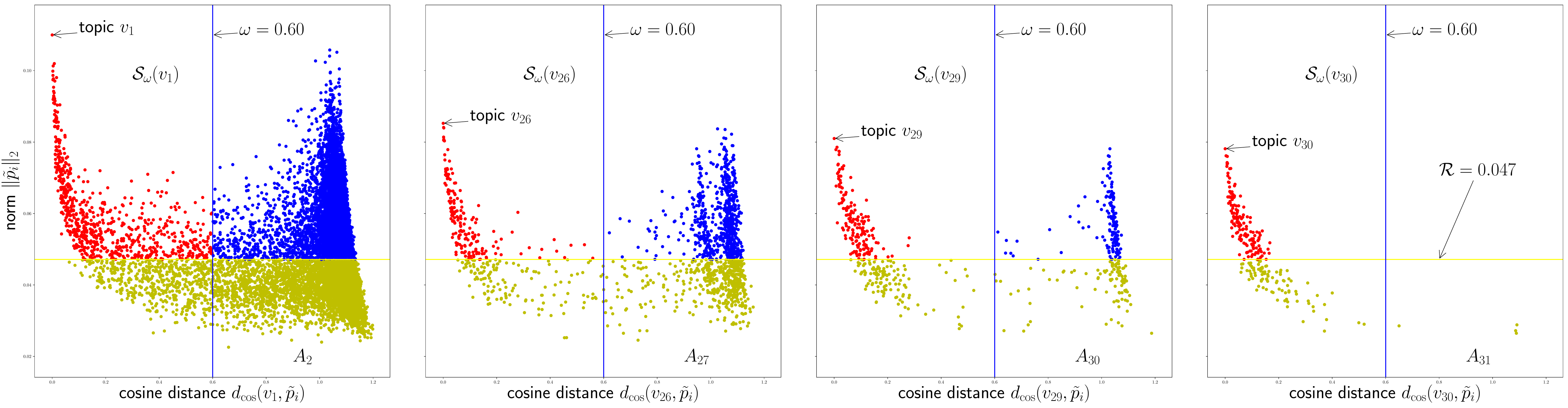}}
\caption{Iterations 1, 26, 29, 30 of the Algorithm \ref{angle_simple}. Red are the documents in the cone $\mathcal{S}_\omega(v_k)$; blue are the documents in the active set $A_{k+1}$ for next iteration. Yellow are  documents $\|\pt_m\|_2<\mathcal{R}$.}
\label{fig:algo_iter}
\end{center}
\vskip -0.2in
\end{figure*}

\begin{algorithm}[ht]
\caption{CoSAC for documents}
\label{angle_full}
\begin{algorithmic}[1]
\REQUIRE normalized documents $\bar w_1,\ldots,\bar w_M$, \\
angle threshold $\omega$, norm threshold $\mathcal{R}$, outlier threshold $\lambda$
\ENSURE topics $\beta_1,\ldots,\beta_k$
\STATE $\hat C_p = \frac{1}{M}\sum_m \bar w_m $ \COMMENT{find center};\qquad $\wt_m := \bar w_m - \hat C_p$ for $m=1,\ldots,M$ \COMMENT{center the data}
\STATE $A_1=\{1,\ldots,M\}$ \COMMENT{initialize active set}; \qquad $k=1$ \COMMENT{initialize topic count}
\WHILE{$\exists \ m \in A_k: \|\wt_m\|_2 > \mathcal{R}$}
\STATE $v_k = \argmax\limits_{\wt_m: m \in A_k}\|\wt_m\|_2$ \COMMENT{initialize direction}
\WHILE[mean-shifting]{$v_k$ not converged}
\STATE $\mathcal{S}_\omega(v_k) = \{m:d_{\cos}(\wt_m,v_k)<\omega\}$ \COMMENT{find cone of near documents}
\STATE $v_k = \sum_{m\in \mathcal{S}_\omega(v_k)}\wt_m/\card(\mathcal{S}_\omega(v_k))$ \COMMENT{update direction}
\ENDWHILE
\STATE $A_k = A_k \setminus \mathcal{S}_\omega(v_k)$ \COMMENT{update active set}\\
\algorithmicif{ $\ \card(\mathcal{S}_\omega(v_k))>\lambda M$\qquad}\algorithmicthen{ $\ k=k+1$ }\COMMENT{record topic direction}
%\IfThenElse{$\card(\mathcal{S}_\omega(v_k)$}{$k=k+1$ }
%\IF{$\card(\mathcal{S}_\omega(v_k))>\lambda M$}\STATE $k=k+1$ \COMMENT{record direction}\ENDIF
%\STATE $R_k := \max\limits_{i:i \in \mathcal{S}_k} <v_k, \bar w_i>$ \COMMENT{find topic length along direction $v_k$}

%\STATE $A_k = A_{k-1} \setminus \mathcal{S}_\omega(v_{k-1})$ \COMMENT{update active set}
%\ELSE
%\STATE $A_k = A_k \setminus \mathcal{S}_\omega(v_k)$ \COMMENT{only update active set}

\ENDWHILE
\STATE $v_1,\ldots,v_k = $ weighted spherical k-means $(v_1,\ldots,v_k,\wt_1,\ldots,\wt_M)$
\FOR{$l$ in $\{1,\ldots,k\}$}
\STATE $R_l := \max\limits_{m:m \in S^l(v_l)} \langle v_l, \wt_m \rangle$ \COMMENT{find topic length along direction $v_l$}
\STATE $\beta_l = R_l v_l + \hat C_p$ \COMMENT{compute topic}
\ENDFOR
\end{algorithmic}
\end{algorithm}

\section{Experimental results}
\label{experiment}
\subsection{Simulation experiments}
%\begin{figure*}[t!]
%\vskip -0.1in
%\begin{center}
%\centerline{\includegraphics[width=\textwidth]{f1.pdf}}
%\caption{Minimum matching Euclidean distance for (a) varying corpora sizes, (b) varying length of documents. (c) Estimation of number of topics}
%\label{fig:mm}
%\end{center}
%\vskip -0.2in
%\end{figure*}
%
%\begin{figure*}[t!]
%\vskip -0.1in
%\begin{center}
%\centerline{\includegraphics[width=\textwidth]{f2.pdf}}
%\caption{Perplexity for (a) varying corpora sizes, (b) varying length of documents. (c) Running times for varying corpora sizes}
%\label{fig:pp}
%\end{center}
%\vskip -0.2in
%\end{figure*}

\begin{figure*}[t!]
\vskip -0.1in
\begin{center}
\centerline{\includegraphics[width=\textwidth]{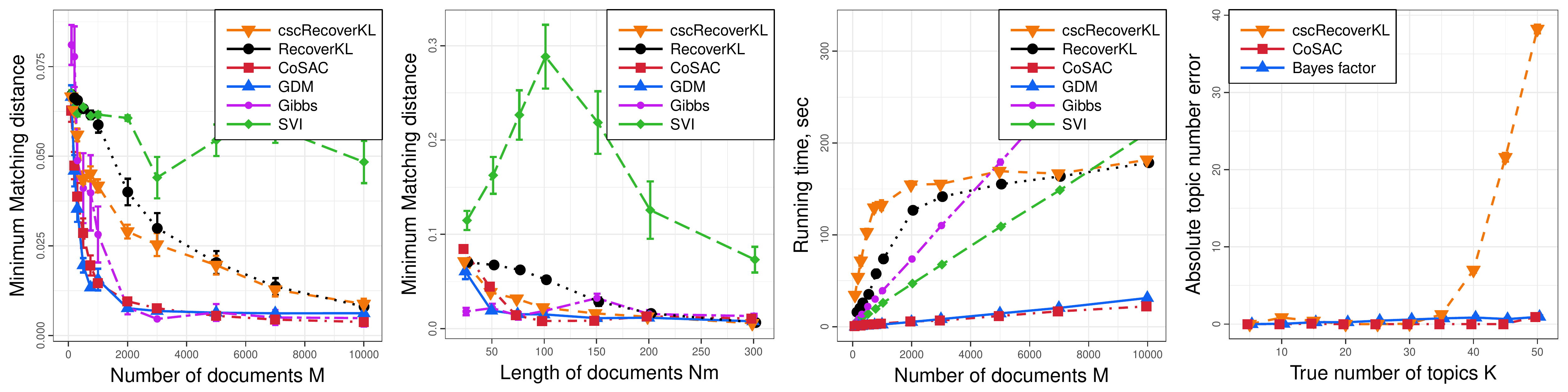}}
\caption{Minimum matching Euclidean distance for (a) varying corpora size, (b) varying length of documents; (c) Running times for varying corpora size; (d) Estimation of number of topics.}
\label{fig:all}
\end{center}
\vskip -0.2in
\end{figure*}

\begin{figure*}[t!]
\vskip -0.1in
\begin{center}
\centerline{\includegraphics[width=\textwidth]{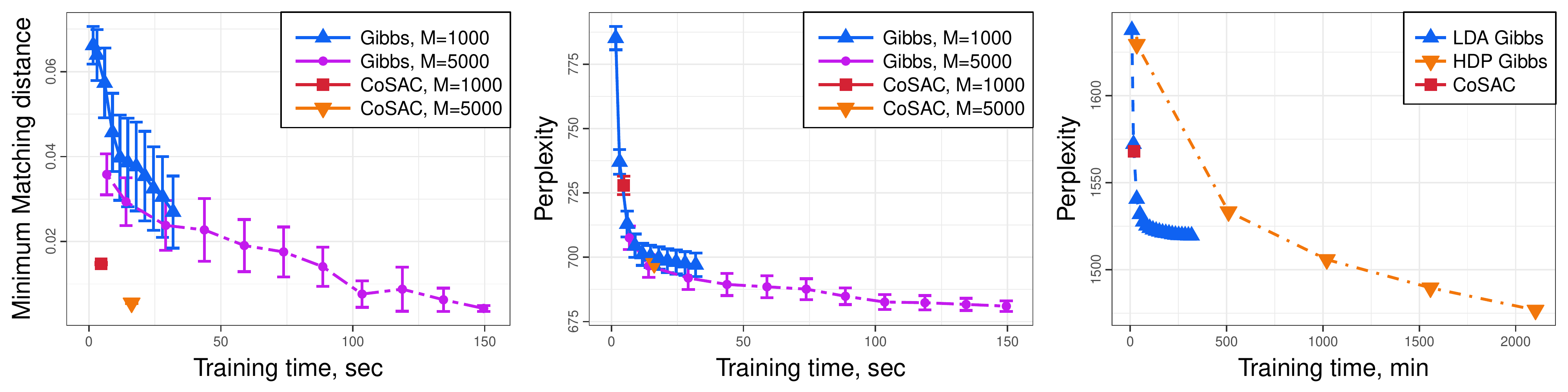}}
\caption{Gibbs sampler convergence analysis for (a) Minimum matching Euclidean distance for corpora sizes 1000 and 5000; (b) Perplexity for corpora sizes 1000 and 5000; (c) Perplexity for NYTimes data.}
\label{fig:time}
\end{center}
\vskip -0.2in
\end{figure*}

%\mikcomm{I will compress simulations, but suggest to keep all the plots}
%We have proposed a new geometric Cone Shifting algorithm alongside a series of theoretical results analyzing geometrical and distributional properties of the topic simplex $B$. To this point our analysis was mostly limited to the noiseless case, which in practice is only achieved when documents are sufficiently long. Nonetheless we hypothesize that even for short documents our results remain meaningful when number of documents $M$ is sufficient - the intuition behind this assertion is based on the Law of Large Numbers (LLN): Algorithm \ref{angle_full} manipulates averages of documents in the corresponding conical regions and so does the post processing spherical k-means step, while LLN guarantees us that averages of noisy observations converge to corresponding noiseless means. 
In the simulation studies we shall compare CoSAC (Algorithm \ref{angle_full}) and cscRecoverKL based on Algorithm \ref{angle_simple} both of which don't have access to the true $K$,  versus popular parametric topic modeling approaches (trained with true $K$): Stochastic Variational Inference (SVI), Collapsed Gibbs sampler, RecoverKL and GDM (more details in the \hyperref[supp:details]{Supplement}). The comparisons are done on the basis of minimum-matching Euclidean distance, which quantifies distance between topic simplices \citep{tang2014understanding}, and running times
(perplexity scores comparison is given in Fig. \ref{fig:supp_all}). Lastly we will demonstrate the ability of CoSAC to recover correct number of topics for a varying $K$.
%Indeed the two scenarios describe the two common topic modeling realities - fewer but longer documents (e.g. scientific articles, novels, legal documents) and larger corpora of shorter documents (e.g. news articles, social media posts). 
\paragraph{Estimation of the LDA topics}
First we evaluate the ability of CoSAC and cscRecoverKL to estimate topics $\beta_1,\ldots,\beta_K$, fixing $K=15$. Fig. \ref{fig:all}(a) shows performance for the case of fewer $M\in [100,10000]$ but longer $N_m=500$ documents (e.g. scientific articles, novels, legal documents). CoSAC demonstrates performance comparable in accuracy to Gibbs sampler and GDM.
%To compute held-out perplexity for the CoSAC we employed projection based estimates for topic proportions $\theta_m$ from \cite{yurochkin2016geometric} (based on Eq. \ref{gdm_obj}), which led to a slightly worse perplexity scores for CoSAC and GDM in comparison to Gibbs sampler.

Next we consider larger corpora $M=30000$ of shorter $N_m \in [25,300]$ documents (e.g. news articles, social media posts). Fig. \ref{fig:all}(b) shows that this scenario is harder and CoSAC matches the performance of Gibbs sampler for $N_m \geq 75$. Indeed across both experiments CoSAC only made mistakes in terms of $K$ for the case of $N_m=25$, when it was underestimating on average by 4 topics and for $N_m=50$ when it was off by around 1, which explains 
the earlier observation. Experiments with varying $V$ and $\alpha$ are given in the \hyperref[supp:exp]{Supplement}.

It is worth noting that cscRecoverKL appears to be strictly better than its predecessor. This suggests that our procedure for selection of anchor words is more accurate in addition to being nonparametric.

\paragraph{Running time}
A notable advantage of the CoSAC algorithm is its speed.
% because computationally it is simply based on evaluating scalar products between sparse documents.\footnote{LINK to code will be available after revision period}
In Fig. \ref{fig:all}(c) we see that Gibbs, SVI, GDM and CoSAC all have linear complexity growth in $M$, but the slopes are very different
and approximately are $IN_m$ for SVI and Gibbs (where $I$ is the number of iterations which has to be large enough for convergence), number of k-means iterations to converge for GDM and is of order $K$ for the CoSAC procedure making it the fastest algorithm of all under consideration.

Next we compare CoSAC to per iteration quality of the Gibbs sampler trained with 500 iterations for $M=1000$ and $M=5000$. Fig. \ref{fig:time}(b) shows that Gibbs sampler, when true $K$ is given, can achieve good perplexity score as fast as CoSAC and outperforms it as training continues, although Fig. \ref{fig:time}(a) suggests that much longer training time is needed for Gibbs sampler to achieve good topic estimates and small estimation variance.

\paragraph{Estimating number of topics}
Model selection in the LDA context is a quite challenging task and, to the best of our knowledge, there is no "go to" procedure. One of the possible approaches is based on refitting LDA with multiple choices of $K$ and using Bayes Factor for model selection \citep{griffiths2004finding}. Another option is to adopt the Hierarchical Dirichlet Process (HDP) model, but we should understand that it is not a procedure to estimate $K$ of the LDA model, but rather a particular prior on the number of topics, that assumes $K$ to grow with the data. A more recent suggestion is to slightly modify LDA and use Bayes moment matching \citep{hsu2016online}, but, as can be seen from Figure 2 of their paper, estimation variance is high and the method is not very accurate (we tried it with true $K=15$ and it took above 1 hour to fit and found 35 topics). Next we compare Bayes factor model selection versus CoSAC and cscRecoverKL
%with $M=5000$ documents and $N_m=500$ words per document and 
for $K\in[5,50]$.
%Average \emph{absolute} error
%over twenty repetitions
%is reported in Fig. \ref{fig:mm}(c). 
Fig. \ref{fig:all}(d) shows that CoSAC consistently recovers \emph{exact} number of topics in a wide range.
%(triple $\omega,\mathcal{R},n$ are set in a same way as in the previous experiments). Indeed across all our experiments CS only made mistakes in terms of $K$ for the case of $N_m=25$, when it was underestimating by 4 topics and for $N_m=50$ it was off by only 1.

We also observe that cscRecoverKL does not estimate $K$ well (underestimates) in the higher range. This is expected because cscRecoverKL finds the number of anchor words, \emph{not} topics. The former is decreasing when later is increasing. Attempting to fit RecoverKL with more topics than there are anchor words might lead to deteriorating performance and our modification can address this limitation of the RecoverKL method.

\subsection{Real data analysis}
In this section we demonstrate CoSAC algorithm for topic modeling on one of the standard bag of words datasets --- NYTimes news articles.  %\footnote{\url{https://archive.ics.uci.edu/ml/datasets/bag+of+words}} 
After preprocessing we obtained $M \approx 130,000$ documents over $V=5320$ words. Bayes factor for the LDA selected the smallest model among $K\in[80,195]$, while CoSAC selected 159 topics. We think that disagreement between the two procedures is attributed to the misspecification of the LDA model when real data is in play, which affects Bayes factor, while CoSAC is largely based on the geometry of the topic simplex.
%Majority of our theoretical results do not assume Dirichlet distribution as in the LDA and remaining ones can be generalized to different continuous density functions, placing nonzero mass on the corners of the topic simplex. 

The results are summarized in Table \ref{nyt} --- CoSAC found 159 topics in less than 20min; cscRecoverKL estimated the number of anchor words in the data to be 27 leading to fewer topics. Fig. \ref{fig:time}(c) compares CoSAC perplexity score to per iteration test perplexity of the LDA (1000 iterations) and HDP (100 iterations) Gibbs samplers. Text files with top 20 words of all topics are \href{https://github.com/moonfolk/Geometric-Topic-Modeling/tree/master/nytimes_topics}{available on GitHub}. We note that CoSAC procedure recovered meaningful topics, contextually similar to LDA and HDP (e.g. elections, terrorist attacks, Enron scandal, etc.) and also recovered more specific topics about Mike Tyson, boxing and case of Timothy McVeigh which were present among HDP topics, but not LDA ones. We conclude that CoSAC is a practical procedure for topic modeling on large scale corpora able to find meaningful topics in a short amount of time.
\begin{table}[]
\caption{Modeling topics of NYTimes articles}
\label{nyt}
%\vskip 0.15in
\centering
\begin{tabular}{lrrrr}
\toprule
{} &  $K$ &  Perplexity &  Coherence &  Time  \\
\midrule
cscRecoverKL   & 27 &   2603 &          -238 &            37 min \\
HDP Gibbs & $221\pm 5$ &   $1477\pm 1.6$ &          $-442\pm 1.7$ &            35 hours  \\
LDA Gibbs & 80 &   $1520\pm 1.5$ &          $-300\pm 0.7$ &            5.3 hours  \\
CoSAC & 159 &   1568 &          -322 &            19 min  \\
\bottomrule
\end{tabular}
\vskip -0.1in
\end{table}

\section{Discussion}
\label{discussion}
We have analyzed the problem of estimating topic simplex without assuming number of vertices (i.e., topics) to be known. We showed that it is possible to cover topic simplex using two types of geometric shapes, cones and a sphere, leading to a
class of Conic Scan-and-Cover algorithms. We then proposed several geometric correction techniques to account for the noisy data. Our procedure is accurate in recovering the true number of topics, while remaining practical due to its computational speed. We think that angular geometric approach might allow for fast and elegant solutions to other clustering problems, although as of now it does not immediately offer a unifying problem solving framework like MCMC or variational inference. An interesting direction in a geometric framework is related to building models based on geometric quantities such as distances and angles.

\subsubsection*{Acknowledgments}
This research is supported in part by grants NSF CAREER DMS-1351362, NSF CNS-1409303, a research
gift from Adobe Research and a Margaret and Herman Sokol Faculty Award.

\appendix
\section{Proofs of main theorems}
\label{supp:proofs}
We start by reminding the reader of our geometric setup. First, topic simplex $B:=\conv(\beta_1,\ldots,\beta_K)$ is centered at a point denoted by $C_p$.
%$C_p:= \E p_1 \in \Delta^{V-1}$ --- the population mean of the data. 
Let $\Delta^{V-1}_0 := \{x\in\mathbb{R}^V: x + C_p \in \Delta^{V-1}\}$ --- centered probability simplex. Then,
write $b_k := \beta_k - C_p \in \Delta^{V-1}_0$ for $k=1,\ldots,K$ and $\pt_m := p_m - C_p \in \Delta^{V-1}_0$ for $m=1,\ldots,M$. Note that re-centering leaves corresponding barycentric coordinates $\theta_m\in\Delta^{K-1}$ unchanged.
Moreover, the extreme points of centered topic simplex $\Bt := \conv\{b_1,\ldots,b_K\}$ can now be represented by 
their directions $v_k\in\mathbb{R}^V$ and corresponding radii $R_k \in \mathbb{R}_+$ such that 
$b_k = R_k v_k$ for any $k = 1,\ldots, K$.
\subsection{Coverage of the topic simplex }
Suppose that $C_p$ is the incenter of the topic simplex $\Bt$, with $r$ being the inradius. Recall
that the incenter and inradius correspond to the maximum volume sphere inside $\Bt$. Let $a_{i,k}$ denote the distance between the $i^{th}$ and $k^{th}$ vertex of $\Bt$, with  $a_{min} \leq a_{i,k} \leq a_{max}$ for all $i,k$, and $R_{max},R_{min}$ such that $R_{min}\leq R_k := \|b_k\|_2 \leq R_{max} \ \forall \ k=1,\ldots,K$
\begin{propn}[\ref{delta_exist}]
For simplex $\Bt$ and
$\omega \in (\omega_1,\omega_2)$, where $\omega_1 = 1- r/ R_{max}$ 
and $\omega_2 = \max \{ (a_{max}^2)/(2R_{max}^2),\max\limits_{i,k=1,\ldots,K}(1-\cos(b_i,b_k)\}$,
the cone $\mathcal{S}_\omega(v)$ around any vertex direction $v$ of $\Bt$ contains exactly one vertex. 
Moreover, complete coverage holds: 
$\bigcup\limits_{k=1}^K \mathcal{S}_{\omega}(v_k) \supseteq \Bt$.
\end{propn}

\begin{proof}
Let $\omega_0= \frac{a_{min}^2}{2R_{max}^2}$. Then, for any $k \in \{1,\dots, K\} $, for any $\omega \leq \omega _0$, $\mathcal{S}_{\omega}(v_k)$ does not contain any other vertices. This can be explained as follows. Fix $k$, and choose $i \in \{1,\dots,K \}\neq k$. Define $\phi_{i,k}$ as the angle at $C_p$ made by the side connecting the vertex $i$ and vertex $k$. Then from the cosine law for triangles, we have
\begin{equation}
\cos(\phi_{i,k}) =\frac{R_i^2 +R_k^2 - a_{i,k}^2}{2R_iR_k}.\nonumber
\end{equation}
Now, for any $\phi \leq \underset{i,k}{\operatorname{min}} \, \phi_{i,k}$, with $\omega_{\phi}= 1- \cos(\phi)$, the cone  $\mathcal{S}_{\omega_{\phi}}(v_k)$ does not cover any vertex other than vertex $k$, for any $k$. Now $\phi_1= \underset{i,k}{\operatorname{min}} \phi_{i,k}$ satisfies
\begin{eqnarray}
1- \cos (\phi_1) \leq \frac{a_{min}^2}{2R_{max}^2} - \frac{(R_{max}-R_{min})^2}{2R_{max}R_{min}} \leq \frac{a_{min}^2}{2R_{max}^2}. \nonumber
\end{eqnarray}
from which we obtain the upper bound for $\omega$. For the lower bound, consider for vertex $k$, $\mathcal{S}(v_k)$ the cone connecting the incenter to facial incenters of facets containing vertex $k$. Then $\bigcup\limits_{k=1}^K \mathcal{S}(v_k) \supseteq \Bt$.
Now for each $k$,  $\mathcal{S}(v_k) \subseteq \mathcal{S}_{\omega_2}(v_k) $, where $\omega_2=1-\cos(\phi_2)$, with $\phi_2$ satisfying $\cos(\phi_2)\leq \underset{k \in \{1,\dots, K\}}{\operatorname{min}}\frac {r}{R_k}$.
From this we get the lower bound. The restriction $2R_{max}^2 \leq a_{min}^2$ is needed to ensure that the set $\big\{ \omega : 1-(\frac{r}{R_{max}})\leq \omega \leq (\frac{a_{max}^2}{2R_{max}^2}) \big\}$ is non-empty.
\end{proof}

\begin{propn}[\ref{r_delta}]
Let $\mathscr{B}(C_p,\mathcal{R})=\{ \tilde{p} \in \mathbb{R}^{V} |  \|\tilde{p} -C_p\|_2 \leq \mathcal{R} \}$, $\mathcal{R} > 0$; $\omega_1,\omega_2$ given in Prop. \ref{delta_exist}, and 
\begin{equation}
\label{supp:w3}
\omega_3 := 1- \min\bigg\{\underset{i,k}\min\left( \frac{R_k \sin^2(b_i,b_k)}{\mathcal{R}} + \cos(b_i,b_k)\sqrt{1-\frac{R_k^2 \sin^2(b_i,b_j)}{\mathcal{R}^2}}\right),1\bigg\},
\end{equation}
then we have 
$\bigcup\limits_{k=1}^K \mathcal{S}_{\omega}(v_k) \cup {\mathscr{B}(C_p,\mathcal{R})} \supseteq \Bt$ whenever $\omega \in (\min\{\omega_1, \omega_3\},\omega_2)$.
\end{propn}

\begin{figure}
\centering 
 \begin{tikzpicture}
    \coordinate (A) at (1,3);
    \coordinate (B) at ($(A)+(2:3)$);
    \coordinate (C) at ($(A)+(85:5)$);
    \draw (A) node[left]{$A$} -- node[below]{$\mathcal{R}$} (B) node[right]{$B$} -- (C)node[above]{$C$}node[midway,above]{$R_k$} -- (A);

    \tikzAngleOfLine(B)(C){\AngleStart}
    \tikzAngleOfLine(B)(A){\AngleEnd}
    \draw[red,<->] (B)+(\AngleStart:2cm) arc (\AngleStart:\AngleEnd:2 cm);
    \node[circle] at ($(B)+({(\AngleStart + \AngleEnd)/2}:1 cm)$) {$\arccos(1-\omega_R)$};
    \tikzAngleOfLine(C)(A){\AngleStart}
    \tikzAngleOfLine(C)(B){\AngleEnd}
    \draw[red,<->] (C)+(\AngleStart:2cm) arc (\AngleStart:\AngleEnd:2 cm);
    \node[circle] at ($(C)+({(\AngleStart+\AngleEnd)/2}:1.6 cm)$) { $\angle(b_i,b_k)$};
\end{tikzpicture}
\caption{ $C$ :  $k^{th}$ vertex point, $A$ : point where the adjacent side to the vertex has been cut off by the sphere, $R_k$: distance to $k^{th}$ vertex from incenter, $\mathcal{R}$ : radius of sphere, $B$ : incenter}
\label{fig:R_delta}
\end{figure}
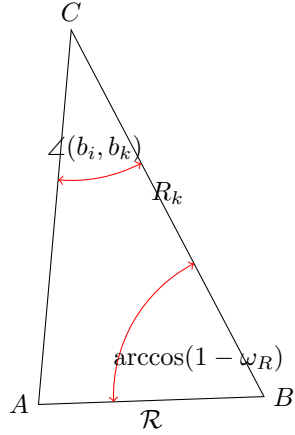

\begin{proof}
Let $\phi_{i,k}=\arccos(1-\omega_{i,k})$ be the angle formed by the line joining the $k^{th}$ vertex to the incenter $C_p$ and the radial vector from incenter to the point where the sphere cuts the edge connecting $i$ and $k$ (segment $AB$ on Fig. \ref{fig:R_delta}). From the sine law for a triangle we have
\begin{eqnarray}
\cos(\phi_{i,k}) +\cot(b_i,b_k) \sin(\phi_{i,k}) - \frac{R_k}{\mathcal{R}}=0.
\end{eqnarray}
Solving for $\phi_{i,k}$ we have $\cos(\phi_{i,k}) =\left( \frac{R_k \sin^2(b_i,b_k)}{\mathcal{R}} + \cos(b_i,b_k)\sqrt{1-\frac{R_k^2 \sin^2(b_i,b_k)}{\mathcal{R}^2}}\right) $. Now, since we must choose the largest such $\phi$ over all $i$ and $k$, the bound follows immediately. Notice that as $\mathcal{R} \rightarrow R_{max}$, the value of $\left( \frac{R_k \sin^2(b_i,b_k)}{\mathcal{R}} + \cos(b_i,b_k)\sqrt{1-\frac{R_k^2 \sin^2(b_i,b_k)}{\mathcal{R}^2}}\right)  \rightarrow 1$, whereas $\frac{r}{R_{max}} < 1$ strictly. Thus, as $\mathcal{R}$ increases the lower bound in this limiting scenario is  dominated by  $1- \underset{i,k}{\operatorname{min}}\left( \frac{R_k \sin^2(b_i,b_k)}{\mathcal{R}} + \cos(b_i,b_k)\sqrt{1-\frac{R_k^2 \sin^2(b_i,b_k)}{\mathcal{R}^2}}\right)$, thereby obtaining an improvement in the bound from Proposition \ref{delta_exist}.
\end{proof}

\begin{propn}[\ref{n_eq}]
The cone $S_\omega(v_1)$ whose axis is a topic direction $v_1$ has mass
\begin{eqnarray}
\left.\begin{aligned}
& \P(\mathcal{S}_\omega(v_1)) > \P(\Lambda_{c}(b_1)) = \frac{\int_{1-c}^1\theta_1^{\alpha_1-1}(1- \theta_1)^{\sum_{i\neq 1}\alpha_i -1}\mathrm{d}\theta_1}{\int_{0}^1\theta_1^{\alpha_1-1}(1- \theta_1)^{\sum_{i\neq 1}\alpha_i -1}\mathrm{d}\theta_1} = \\
& \frac{c^{\sum_{i \neq 1}\alpha_i}(1-c)^{\alpha_1}\Gamma(\sum_{i=1}^K \alpha_i)}{(\sum_{i \neq 1} \alpha_i) \Gamma(\alpha_1) \Gamma(\sum_{i \neq 1} \alpha_i)}
\biggr [1 + \frac{c\sum_{i=1}^K \alpha_i }{\sum_{i \neq 1} \alpha_i + 1}  + \frac{c^2(\sum_{i=1}^K \alpha_i)(\sum_{i=1}^K \alpha_i +1) }{(\sum_{i \neq 1} \alpha_i + 1)(\sum_{i \neq 1} \alpha_i + 2)} + \cdots  
\biggr ],
\end{aligned}\right.
\end{eqnarray}
where $\Lambda_{c}(b_1)$ is the simplicial cap of $\mathcal{S}_\omega(v_1)$  which is composed
of vertex $b_1$ and a base parallel to the corresponding base of $\Bt$ and cutting adjacent edges of $\Bt$ in the ratio $c:(1-c)$. 
\end{propn}
The truncated beta probability calculations in Proposition \ref{n_eq} can be found in \cite{incompletebeta}.

\begin{propn}[\ref{n_theorem}]
For $\lambda\in (0,1)$, let  $c_\lambda$ be such that $\lambda =\underset{k} \min\; \P(\Lambda_{c_{\lambda}}(b_k))$ and let $\omega_\lambda$ be such that
\begin{equation}
\label{supp:c_lambda}
c_\lambda=\left(\left(2\sqrt{1- \frac{r^2}{R_{max}^2}}\right) \  \left( \sin(d)\cot(\arccos(1-\omega_\lambda)) + \cos(d)  \right) \right)^{-1},
\end{equation}
where angle $d \leq \underset{i,k}{\operatorname{min}}\, \angle(b_k, b_k - b_i)$. Then, as long as
%where $d$ satisfies $\cos(d) \geq \underset{i,k}{\operatorname{min}}\frac{R_{k}^2 + a_{i,k}^2 -R_{i}^2}{2R_{k}a_{i,k}} \geq\frac{R_{min}^2 + a_{min}^2 -R_{max}^2}{2R_{max}a_{max}}$. Then,
% as long as
\begin{equation}
\label{supp:omega_lambda}
\omega\in\left(\omega_\lambda, \max\left(\frac{a_{max}^2}{2R_{max}^2},\max\limits_{i,k=1,\ldots,K}(1-\cos(b_i,b_k)\right)\right),
\end{equation} 
the bound $\P(\mathcal{S}_{\omega}(v_k)) \geq \lambda\text{ holds for all }k=1,\ldots, K$.
\end{propn}

\begin{proof}
Consider Figure \ref{fig:R_delta}, with length of $AC = a_{i,k}c$, where $c$ is the proportion in which the cone cuts $AC$, the edge joining vertex $i$ and vertex $k$. Now, from the sine law of a triangle,
\begin{eqnarray}\label{c_n dependence}
\frac{R_k}{a_{i,k}c}=\sin(b_i,b_k) \cot{\phi_{i,k}} + \cos(b_i,b_k)
\end{eqnarray}
where $\phi_{i,k}$ is as defined in the proof of Proposition \ref{r_delta}. Now $\frac{a_{i,k}}{R_k}\leq \frac{2(\sqrt{R_{max}^2-r^2})}{R_{max}}$. The  choice of $\phi_{\lambda}= \cos{\omega_{\lambda}}$ satisfies
\begin{eqnarray}
c_\lambda \geq \frac{1}{2\sqrt{1-\frac{r^2}{R_{max}^2}}} \underset{i,k}{\operatorname{min}}\frac{1}{\sin(b_i,b_k) \cot{\phi_\lambda} + \cos(b_i,b_k)}
\end{eqnarray}
therefore proves the theorem. Since, $\phi_{\lambda}\leq \frac{\pi}{2} - \angle(b_i,b_k)$, for all $i,k$,  the function $\sin(b_i,b_k) \cot{\phi_{\lambda}} + \cos(b_i,b_k)$ is increasing as  the angle between $b_i$ and $b_k$ increases, as can be checked  for maxima by the first derivative rule. Using the cosine law,
\begin{eqnarray}
\cos(b_i,b_k)=\frac{-R_i^2 +R_k^2 + a_{i,k}^2}{2a_{i,k}R_k}.
\end{eqnarray}
Minimizing this quantity with respect to $i$ and $k$ we get the result.
\end{proof}

\subsection{Consistency of the Conic Scan-and-Cover algorithm}
Under the LDA setup (as presented  in  Section \ref{back}), recall that $a_{i,k}$ is the length of the edge connecting the $i^{th}$ and $k^{th}$ vertex, i.e., $\|\beta_i -\beta_k\|_2 = a_{i,k}$, where $\| \cdot \|_2$ is the $\ell_2$ norm. Let $\mathscr{B}(\cdot,\mathcal{\epsilon})$ denote an $\epsilon$-ball in $\ell_2$-norm.
Then the following result states that with high probability there exists a document in a neighborhood of every vertex.

\begin{lem}\label{single document}
Let $p_m := \sum_k \beta_k \theta_{mk}$ for $m=1,\ldots,M$ as before. Then for any $i$ and any  $0 < \epsilon < \underset{k\neq i}{\operatorname{max}} \, a_{i,k}$, 
\begin{eqnarray}
\P( p_m \notin \mathscr{B}(\beta_i,\mathcal{\epsilon}) \; \forall \; m \in \{1,\dots,M\}) \leq \left( \frac{\int_{0}^{1- (\epsilon/ \underset{k\neq i}{\operatorname{max}} a_{i,k})}\theta_i^{\alpha_i-1}(1- \theta_i)^{\sum_{j\neq i}\alpha_j -1}\mathrm{d}\theta_i}{\int_{0}^1\theta_i^{\alpha_i-1}(1- \theta_i)^{\sum_{j\neq 1}\alpha_j -1}\mathrm{d}\theta_i} \right)^{M}.
\end{eqnarray}
\end{lem}

Since $\left( \frac{\int_{0}^{1- (\epsilon/\underset{k\neq i}{\operatorname{max}} a_{i,k})}\theta_i^{\alpha_i-1}(1- \theta_i)^{\sum_{j \neq i}\alpha_j -1}\mathrm{d}\theta_i}{\int_{0}^1\theta_i^{\alpha_i-1}(1- \theta_i)^{\sum_{j\neq 1}\alpha_j -1}\mathrm{d}\theta_i} \right) < 1$, for all $i$ because Beta distribution is absolutely continuous in $(0,1)$, the bound on the right hand side  goes to $0$ as $M \rightarrow \infty$.

Let $\{\hat{\beta_1},\dots, \hat{\beta_{K}}\}$ be the topics identified by Conic Scan-and-Cover algorithm, 
with labels permuted according to the minimum matching distance criteria, with $\{\beta_1,\dots,\beta_K\}$ 
being the true topics. Then the following result shows the consistency of the identified topics.
%Then clearly, $\|\beta_i-\tilde{\beta_i}\| \leq \|\beta_i- p_m\|$ for any $m \in \{1,\dots,M\}$, $M \rightarrow \infty$.

\begin{thmn}[\ref{consistency}]
Suppose $\{\beta_1,\dots,\beta_K\}$ are the true topics, incenter $C_p$ is given, $\theta_m\thicksim\Dir_K(\alpha)$ and $p_m := \sum_k \beta_k \theta_{mk}$ for $m=1,\ldots,M$ and $\alpha \in \mathbb{R}_{+}^K$. Let $\{\hat \beta_1,\dots,\hat \beta_{\hat{K}}\}$ be the output of the Conic Scan-and-Cover algorithm trained with $\omega$ and $\mathcal{R}$ as in Proposition \ref{r_delta}. Then $\forall \ \epsilon >0$,
\begin{eqnarray*}
\P \left ( \left \{\underset{j \in \{1, \dots, \hat{K}\}}\min\|\beta_i - \hat \beta_j \| > \epsilon \ \text{, for any $i$}\in \ \{1,\dots,\hat{K}\} \right \} \cup\{ K \neq \hat{K}\} \right) \rightarrow 0\text{ as }M\rightarrow \infty.
\end{eqnarray*}
\end{thmn}

\begin{proof}
From the description of the Conic Scan-and-Cover algorithm it suffices to prove 
that for the suitable choice of $\omega, \mathcal{R}$ as in Proposition \ref{r_delta} 
there holds $\P( \exists \ x_i \  \in \ \{p_1,\dots,p_m\} \ \text{such that} \ \|\beta_i - x_i\| < \epsilon \ \forall \ i \in \ \{1,\dots,K\}) \rightarrow 1$ as $M\rightarrow \infty$.
%But, $\P( \exists \ x_i \  \in \ \{p_1,\dots,p_m\} \ \text{such that} \ \|\beta_i - x_i\| < \epsilon \ \forall \ i \in \ \{1,\dots,K\})$
But this probability expression is bounded from below by
$1- \sum_{i=1}^K \P( p_m \notin  \mathscr{B}(\beta_i,\mathcal{\epsilon}) \; \forall \; m \in \{1,\dots,M\})$. 
The conclusion now follows from Lemma {\ref{single document}}.
\end{proof}

\subsection{Variance argument for multinomial setup}
In the topic modeling problem we are not given $p_m$ for $m=1,\ldots, M$. 
Under the bag-of-words assumption we have access to the frequencies of words in documents $w_1,\ldots,w_M$ 
which provide a point estimate $\bar w_m := w_m/N_m$ for the $p_m$.
The following proposition establishes a bound on the variation of $\bar w_m$ from $p_m$.

\begin{prop}
\label{supp:var}
\begin{eqnarray}
\mathbb{E}[\|\bar w_m  - p_m\|_2^2] \leq \frac{1-(1/V)}{N_m}.
\end{eqnarray}
\end{prop}

\begin{proof}
By iterated expectation identity,
\begin{equation}
\begin{split}
\mathbb{E}[\|\bar w_m  - p_m\|_2^2] & = 
\mathbb{E} \biggr [\mathbb{E} \biggr [\sum_{i=1}^V \|\bar w_{mi}  - p_{mi}\|_2^2 \biggr | p_m \biggr ] \biggr ] \\
& = \mathbb{E} \biggr [\sum_{i=1}^V \frac{p_{mi}(1-p_{mi})}{N_m} \biggr ] \\
& = \frac{1- \mathbb{E}[\sum_{i=1}^V p_{mi}^2 ]}{N_m}
 \leq \frac{1-(1/V)}{N_m}. \nonumber
\end{split}
\end{equation}
The second equality follows because conditioned on $p_m$, each  $ w_{mi} \sim \text{Bin}(N_m,p_{mi})$. The last inequality follows from Cauchy-Schwartz Inequality.
\end{proof}

\section{Spherical k-means for topic modeling}
\label{supp:sph}
We aim to clarify the role of Step 11 of the document Conic Scan-and-Cover algorithm, a geometric
correction technique based on weighted spherical k-means optimization. 
%We also suggest that
%the weighted spherical k-means optimization can be used for topic estimation 
%under the equilaterality of the topic simplex and symmetricity of the Dirichlet distribution of document topic proportions assumptions, analogously to the Geometric Dirichlet Means (GDM) algorithm \cite{yurochkin2016geometric}.
\subsection{Topic directions as solutions to weighted spherical k-means}
\label{cvt}
%We start by defining the notations. As in the main text, $\hat C_p = \frac{1}{M}\sum_m p_m$ is the data center and the problem is analyzed with respect to it. Centered topics are $b_k := \beta_k - \hat C_p$ for $k=1,\ldots,K$ and centered document generating distributions $\pt_m := p_m - \hat C_p$ for $m=1,\ldots,M$. Corresponding barycentric coordinates $\theta_m\in\Delta^{K-1}$ remain unchanged. Centered topics can now be represented by their directions $v_k\in\mathbb{R}^V$ and radii $R_k \in \mathbb{R}_+$ such that $b_k = R_k v_k$ for any $k$.
Let centered document norms $r_m := \|\pt_m\|_2$ for $m=1,\ldots,M$ and $\alpha_k(v):= \cos(b_k, v)$,
cosine of the angle between direction $v$ and $k$-th topic.
%Below is the objective function we discussed in our recent meeting.\\
%Recall that we first center the data at $C = \frac{1}{M}\sum_m p_m$. All following notations are given assuming $C$ is the origin.
%\begin{itemize}
%\item $\beta_1,\ldots,\beta_K$ are centered topics.
%\item $\theta^{(1)},\ldots,\theta^{(M)}$ are barycentric coordinates (i.e. topic proportions) of centered documents in the centered topic simplex
%\item $p_j = \sum_k \theta_i^{(j)}\beta_i$ are centered document generating distributions that we observe (we are at the noiseless case so far)
%\item $r_j := ||p_j||_2$ for $i=j,\ldots,M$ are document norms
%\item $R_i := ||\beta_i||_2$ for $i=1,\ldots,K$ are topic norms
%\item $\alpha_{ij} := \cos(\beta_i,\beta_j)$ for $i,j=1,\ldots,K$ are cosines of angles between topics
%\end{itemize}
The weighted spherical k-means objective takes the form
\begin{equation}
\label{sph_means_obj}
\min\limits_{\|v_k\|_2=1, k=1,\ldots K}\sum_{k=1}^K\sum_{m\in S^k(v_k)} r_m(1-\cos(v_k, \pt_m)),
\end{equation}
where $S^k(v_k) := \{m | \cos(v_k, \pt_m)) > \cos(v_l,\pt_m)\ \forall l\neq k)\}$. Next observe that:
\begin{equation}
r_m\cos(v_k, \pt_m) = \langle v_k,\pt_m \rangle = \sum_{i=1}^K \theta_{mi}\langle v_k, b_i \rangle = \sum_{i=1}^K \theta_{mi} R_i\alpha_i(v_k),
\end{equation}
so our objective \ref{sph_means_obj} becomes:
\begin{equation}
\label{obj}
\max\limits_{\|v_k\|_2=1, k=1,\ldots K}\sum_{k=1}^K\sum_{m\in S^k(v_k)} \sum_{i=1}^K \theta_{mi} R_i\alpha_i(v_k).
\end{equation}
Now, if $R_1=\ldots=R_K$ and $\alpha_i(b_k) = \alpha_i(b_l)\ \forall \ k,l\neq i$, which implies that topic simplex is equilateral, we see that cluster boundaries of topic directions are given by $m \in S^k(b_k)$ iff $\theta_{mk} > \theta_{ml} \ \forall \ l\neq k$. Observe that the corresponding partition is defined by the geometric \emph{medians} of topic simplex, which in turn partitions it into equal volume parts. Then, assuming that the topic simplex $B$ is symmetric,
combined with the symmetricity of the Dirichlet distribution of $\theta_m$-s, it follows that $b_k$ is the centroid of $S^k(b_k)$ for $k=1,\ldots,K$.

\subsection{Role of the spherical k-means in CoSAC algorithm for documents}
The result of Section \ref{cvt} shows that weighted spherical k-means with Lloyd type updates \citep{lloyd1982least} will converge to the directions of the true topics if it is initialized in their respective neighborhoods and equilaterality of $B$ and symmetricity of Dirichlet for document topic proportions is satisfied. 

Recall that goal of the Conic Scan-and-Cover is to find the number of topics and their directions, while Mean Shifting was used to address the noise in the data. 
We proceed to compare weighted spherical k-means by itself (with 500 iterations, which makes it slower than CoSAC) versus document Conic Scan-and-Cover with only Mean Shifting and the full document Conic Scan-and-Cover algorithm to see the effect of the spherical k-means post-processing step. Results in Fig. \ref{fig:supp_sph} are for the same scenarios as in Section \ref{experiment} -- that is when either documents are short $N_m\in[25,300]$ but corpora is large $M=30000$ or when documents are longer $N_m=500$ and corpora is smaller $M \in [100, 10000]$. We see that spherical k-means by itself does not succeed, whereas when used as a postprocessing step for CoSAC it allows for a slight improvement when documents are short. This is because it operates on the full data partition when taking averages for direction estimation, while Mean Shifting only has access to the data in its respective cone $\mathcal{S}_\omega(v)$. Using more data is important for noise reduction when documents are short as suggested by our analysis.

\begin{figure*}[ht]
\vskip -0.1in
\begin{center}
\centerline{\includegraphics[width=\textwidth]{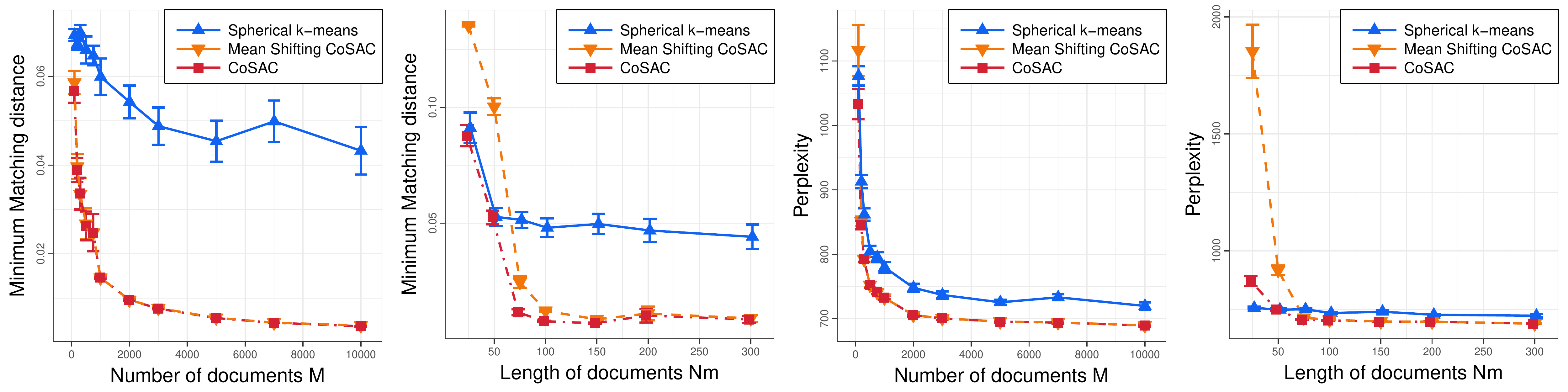}}
\caption{Minimum matching Euclidean distance for (a) varying corpora size, (b) varying length of documents. Perplexity for (c) varying corpora sizes, (d) varying length of documents.}
\label{fig:supp_sph}
\end{center}
\vskip -0.2in
\end{figure*}

\section{Additional experiments}
\label{supp:exp}
\subsection{Perplexity comparison}
In this section we present perplexity scores comparison for the experiments of Section \ref{experiment}. For simulation experiments we used $V=2000$, symmetric $\alpha=\eta=0.1$. To compute held-out perplexity for the CoSAC we employed projection based estimates for topic proportions $\theta_m$ from \cite{yurochkin2016geometric}, which led to a slightly worse perplexity scores for CoSAC and GDM in comparison to Gibbs sampler. However, CoSAC (except for $N_m=25$, when it slightly underestimates $K$) shows competitive performance without requiring $K$ as an input. 
We note that as before cscRecoverKL outperforms RecoverKL in all cases.

\subsection{Varying vocabulary size $V$}
Our next experiment investigates the influence of vocabulary size $V$. We set $N_m = 500$, $M = 5000$, $K=15$, symmetric $\alpha=\eta=0.1$ and varied $V$ from 2000 to 15000. We discovered that $\omega = 0.6$ is too small for $V>10000$, meaning that CoSAC algorithm does not find enough documents in the corresponding cones and keeps discarding without recording topics (per Step 9 of Algorithm \ref{angle_full}). This can be explained by the fact that vectors tending to be far apart in high dimensions and relatively (to $V>10000$) small values of corpora size $M$ and document lengths $N_m$. On the other hand, setting $\omega=0.75$ worked well for all values of $V$ in this experiment. Results are reported in Fig. \ref{fig:supp_all}(c), (d) and Fig. \ref{fig:supp_alpha}(d). Document CoSAC with $\omega=0.75$ recovered true $K=15$ for all values of $V$ and showed better recovery than GDM and Gibbs sampler in terms of minimum matching distance, while Gibbs sampler had slightly better perplexity for higher values of $V$. It is worth reminding that unlike CoSAC, both GDM and the Gibbs sampling based method requires the number of topics $K$ be given.

\subsection{Impact of $\alpha$}
Recall that, per the LDA model, topic proportions $\theta \thicksim \Dir_{K}(\alpha)$. Cases with $\alpha > 1$ were previously shown \citep{nguyen2015posterior} to exhibit slower convergence rates of the LDA's posterior estimation (via Gibbs sampler, for instance). Geometrically, large $\alpha$ implies that documents are more likely concentrated near the center of the topic simplex, leaving fewer documents near the vertices; this entails that geometric inference is more challenging. In our choices for parameters $\omega, \mathcal{R}, \lambda$ we relied on small values of $\alpha$ as a more practical scenario. 
%For larger $\alpha$, it is harder to choose a good value for $\mathcal{R}$ as documents concentrate around the center and form more of a sphere, rather than a simplex. 
Specifically, we considered $\omega=0.8$ for this experiment to achieve full coverage of the topic simplex. In our previous experiments we set $\alpha=0.1$. Now, we consider a larger range, $\alpha \in [0.01,1.5]$, to gauge its impact more fully.
Results are reported in Fig. \ref{fig:supp_alpha}(a), (b) and (c). For smaller values of $\alpha$ CoSAC is demonstrated to be the best algorithm of all under consideration. As $\alpha$ increases, CoSAC can still recover correct $K$ with high accuracy, although the quality of topic estimates deteriorates faster than for Gibbs sampler and GDM. We think that further work on estimation procedures for topic radii $R_k$s (recall that topics are estimated as direction and length along this direction $b_k = R_k v_k$) might address this issue. In this work we considered maximum projection (Step 13 of Algorithm \ref{angle_full}) to estimate $R_k$s, which might not be as accurate when documents are mostly near the center of the topic simplex (i.e., for higher $\alpha$).

\begin{figure*}[ht]
\vskip -0.1in
\begin{center}
\centerline{\includegraphics[width=\textwidth]{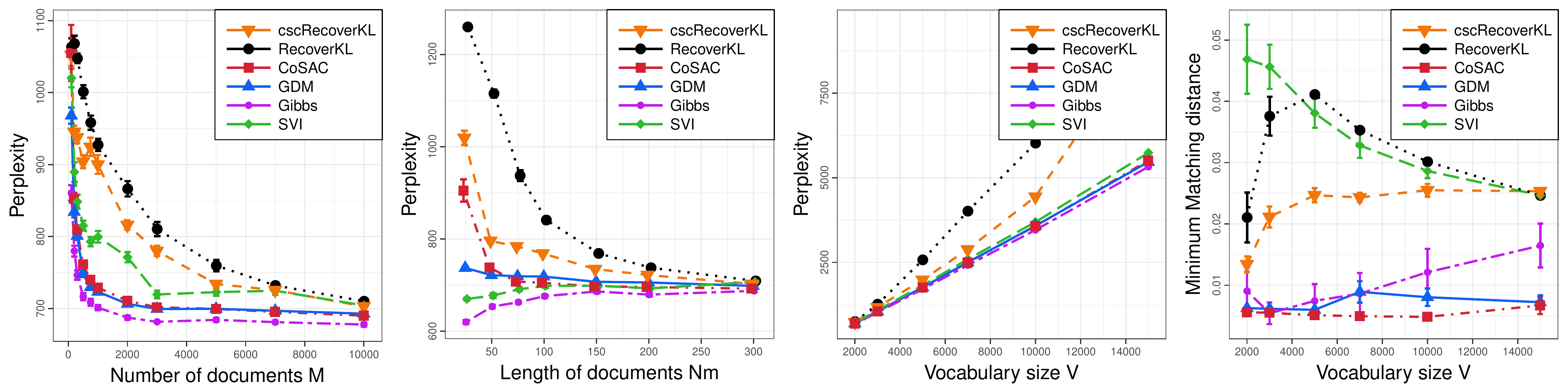}}
\caption{Perplexity for (a) varying corpora size, (b) varying length of documents, (c) varying vocabulary size; (d) Minimum matching Euclidean distance for varying vocabulary size.}
\label{fig:supp_all}
\end{center}
\vskip -0.2in
\end{figure*}

\begin{figure*}[ht]
\vskip -0.1in
\begin{center}
\centerline{\includegraphics[width=\textwidth]{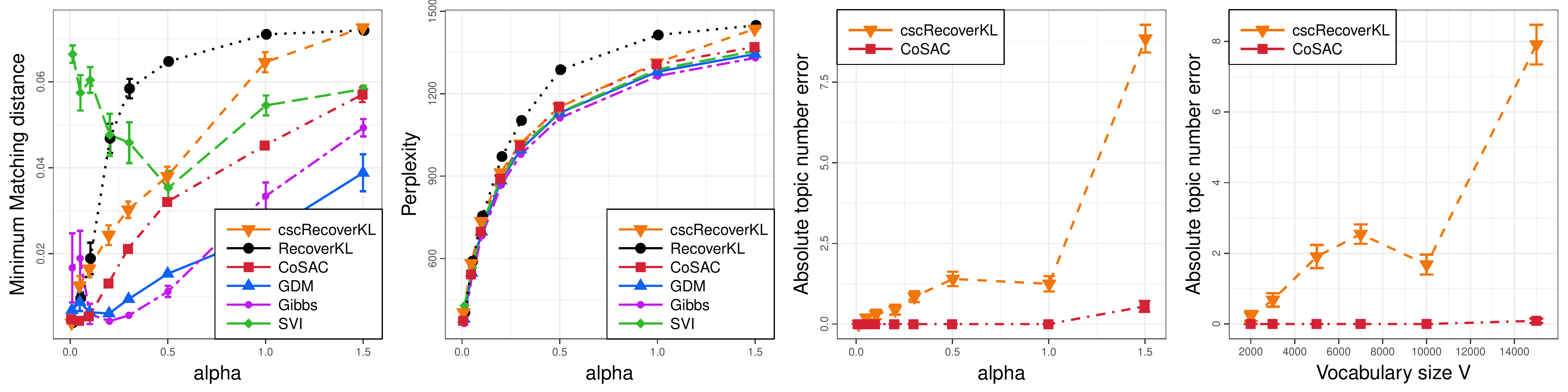}}
\caption{Varying $\alpha$ (a) Minimum matching Euclidean distance, (b) Perplexity, (c) Estimation of number of topics; (d) Estimation of number of topics for varying vocabulary size.}
\label{fig:supp_alpha}
\end{center}
\vskip -0.2in
\end{figure*}

\section{Implementation details}
\label{supp:details}
In this section we give details about the implementations of the algorithms used in simulation studies and real data. We implemented Conic Scan-and-Cover (CoSAC) algorithm in Python with the help of Scipy \citep{scipy} sparse matrix modules. Geometric Dirichlet Means (GDM) \citep{yurochkin2016geometric} was implemented with the help of Scikit-learn \citep{pedregosa2011scikit} k-means implementation (with 10 restarts to avoid local minima of k-means) combined with a geometric correction 
technique. Codes for CoSAC and GDM are available at \url{https://github.com/moonfolk/Geometric-Topic-Modeling}. For RecoverKL \citep{arora2012practical} we applied  \href{https://github.com/MyHumbleSelf/anchor-baggage/tree/master/anchor-word-recovery}{code from one of the coauthors}. To implement cscRecoverKL we used our CoSAC implementation (Algorithm \ref{angle_simple} with outlier threshold $\lambda$ as in Algorithm \ref{angle_full}) to find anchor words and then recovery routine from the aforementioned code. For the Gibbs sampling \citep{griffiths2004finding} we used an \href{https://pypi.python.org/pypi/lda}{lda package} in Python that utilizes Cython to achieve C speed. Gibbs sampler was trained with $\alpha=0.1$, $\eta=0.01$ and 500 iterations in simulations studies and $\alpha=0.1$, $\eta=0.1$ and 1000 iterations in the NYTimes articles\footnote{\url{https://archive.ics.uci.edu/ml/datasets/bag+of+words}} analysis. For the SVI \citep{hoffman2013stochastic} we used Gensim implementation \citep{rehurek_lrec} with automatic hyperparameters estimation, 50 iterations and 10 passes. Finally for HDP \citep{teh2006hierarchical} we used \href{https://github.com/blei-lab/hdp}{C++ implementation} with default hyperparameter settings and 100 iterations. For all experiments (except large vocabulary sizes and bigger $\alpha$), per discussions in Sections \ref{cone_s} and \ref{doc_scan}, parameters of the CoSAC were set to $\omega=0.6,n=0.001M$ and $\mathcal{R}$ as median of the centered and normalized document norms. Spherical k-means post-processing step was run for 30 iterations. For cscRecoverKL we set $\omega=0.4,\lambda=0.015$ ($\lambda=0.005$ for real data) and $\mathcal{R}$ as corresponding median of the norms. Note that cscRecoverKL takes word-to-word co-occurrence matrix as input, therefore sample size is $V$ and "documents" are rows of this matrix. Exploring distributional properties of the simplex spanned by the anchor words is outside of the scope of this work, therefore parameter choices were made empirically based on the visual analysis illustrated by Fig. \ref{fig:algo_iter}. All simulated results are reported after 20 repetitions of the data generation for each scenario and NYTimes results for LDA and HDP are reported over 10 refits of the corresponding Gibbs samplers.

\clearpage
\small
\bibliography{MY_ref}

\begin{thebibliography}{18}
\providecommand{\natexlab}[1]{#1}
\providecommand{\url}[1]{\texttt{#1}}
\expandafter\ifx\csname urlstyle\endcsname\relax
  \providecommand{\doi}[1]{doi: #1}\else
  \providecommand{\doi}{doi: \begingroup \urlstyle{rm}\Url}\fi

\bibitem[Anandkumar et~al.(2012)Anandkumar, Foster, Hsu, Kakade, and
  Liu]{anandkumar2012spectral}
Anandkumar, A., Foster, D.~P., Hsu, D., Kakade, S.~M., and Liu, Y.
\newblock A spectral algorithm for {L}atent {D}irichlet {A}llocation.
\newblock \emph{NIPS}, 2012.

\bibitem[Arora et~al.(2012)Arora, Ge, Halpern, Mimno, Moitra, Sontag, Wu, and
  Zhu]{arora2012practical}
Arora, S., Ge, R., Halpern, Y., Mimno, D., Moitra, A., Sontag, D., Wu, Y., and
  Zhu, M.
\newblock A practical algorithm for topic modeling with provable guarantees.
\newblock \emph{arXiv preprint arXiv:1212.4777}, 2012.

\bibitem[Blei et~al.(2003)Blei, Ng, and Jordan]{blei2003latent}
Blei, D.~M., Ng, A.~Y., and Jordan, M.~I.
\newblock Latent {D}irichlet {A}llocation.
\newblock \emph{J. Mach. Learn. Res.}, 3:\penalty0 993--1022, March 2003.

\bibitem[Deerwester et~al.(1990)Deerwester, Dumais, Furnas, Landauer, and
  Harshman]{deerwester1990indexing}
Deerwester, S., Dumais, S.~T., Furnas, G.~W., Landauer, T.~K., and Harshman, R.
\newblock Indexing by latent semantic analysis.
\newblock \emph{Journal of the American Society for Information Science},
  41\penalty0 (6):\penalty0 391, Sep 01 1990.

\bibitem[Griffiths \& Steyvers(2004)Griffiths and
  Steyvers]{griffiths2004finding}
Griffiths, Thomas~L and Steyvers, Mark.
\newblock Finding scientific topics.
\newblock \emph{PNAS}, 101\penalty0 (suppl. 1):\penalty0 5228--5235, 2004.

\bibitem[Hoffman et~al.(2013)Hoffman, Blei, Wang, and
  Paisley]{hoffman2013stochastic}
Hoffman, Ma.~D., Blei, D.~M., Wang, C., and Paisley, J.
\newblock Stochastic variational inference.
\newblock \emph{J. Mach. Learn. Res.}, 14\penalty0 (1):\penalty0 1303--1347,
  May 2013.

\bibitem[Hsu \& Poupart(2016)Hsu and Poupart]{hsu2016online}
Hsu, Wei-Shou and Poupart, Pascal.
\newblock Online bayesian moment matching for topic modeling with unknown
  number of topics.
\newblock In \emph{Advances In Neural Information Processing Systems}, pp.\
  4529--4537, 2016.

\bibitem[Jones et~al.(2001--)Jones, Oliphant, Peterson, et~al.]{scipy}
Jones, Eric, Oliphant, Travis, Peterson, Pearu, et~al.
\newblock {SciPy}: Open source scientific tools for {Python}, 2001--.
\newblock URL \url{http://www.scipy.org/}.

\bibitem[Lloyd(1982)]{lloyd1982least}
Lloyd, S.
\newblock Least squares quantization in {PCM}.
\newblock \emph{Information Theory, IEEE Transactions on}, 28\penalty0
  (2):\penalty0 129--137, Mar 1982.

\bibitem[Nguyen(2015)]{nguyen2015posterior}
Nguyen, XuanLong.
\newblock Posterior contraction of the population polytope in finite admixture
  models.
\newblock \emph{Bernoulli}, 21\penalty0 (1):\penalty0 618--646, 02 2015.

\bibitem[Olver et~al.(2010)Olver, Lozier, Boisvert, and Clark]{incompletebeta}
Olver, Frank W.~J., Lozier, Daniel~M., Boisvert, Ronald~F., and Clark,
  Charles~W.
\newblock {NIST} handbook of mathematical functions, cambridge university
  press, 2010.
\newblock URL \url{http://dlmf.nist.gov/8.17}.

\bibitem[Pedregosa et~al.(2011)Pedregosa, Varoquaux, Gramfort, Michel, Thirion,
  Grisel, Blondel, Prettenhofer, Weiss, Dubourg, et~al.]{pedregosa2011scikit}
Pedregosa, Fabian, Varoquaux, Ga{\"e}l, Gramfort, Alexandre, Michel, Vincent,
  Thirion, Bertrand, Grisel, Olivier, Blondel, Mathieu, Prettenhofer, Peter,
  Weiss, Ron, Dubourg, Vincent, et~al.
\newblock Scikit-learn: Machine learning in python.
\newblock \emph{Journal of Machine Learning Research}, 12\penalty0
  (Oct):\penalty0 2825--2830, 2011.

\bibitem[Pritchard et~al.(2000)Pritchard, Stephens, and
  Donnelly]{pritchard2000inference}
Pritchard, Jonathan~K, Stephens, Matthew, and Donnelly, Peter.
\newblock Inference of population structure using multilocus genotype data.
\newblock \emph{Genetics}, 155\penalty0 (2):\penalty0 945--959, 2000.

\bibitem[{\v R}eh{\r u}{\v r}ek \& Sojka(2010){\v R}eh{\r u}{\v r}ek and
  Sojka]{rehurek_lrec}
{\v R}eh{\r u}{\v r}ek, Radim and Sojka, Petr.
\newblock {Software Framework for Topic Modelling with Large Corpora}.
\newblock In \emph{{Proceedings of the LREC 2010 Workshop on New Challenges for
  NLP Frameworks}}, pp.\  45--50, Valletta, Malta, May 2010. ELRA.
\newblock \url{http://is.muni.cz/publication/884893/en}.

\bibitem[Tang et~al.(2014)Tang, Meng, Nguyen, Mei, and
  Zhang]{tang2014understanding}
Tang, Jian, Meng, Zhaoshi, Nguyen, Xuanlong, Mei, Qiaozhu, and Zhang, Ming.
\newblock Understanding the limiting factors of topic modeling via posterior
  contraction analysis.
\newblock In \emph{Proceedings of The 31st International Conference on Machine
  Learning}, pp.\  190--198. ACM, 2014.

\bibitem[Teh et~al.(2006)Teh, Jordan, Beal, and Blei]{teh2006hierarchical}
Teh, Y.~W., Jordan, M.~I., Beal, M.~J., and Blei, D.~M.
\newblock Hierarchical dirichlet processes.
\newblock \emph{Journal of the american statistical association}, 101\penalty0
  (476), 2006.

\bibitem[Xu et~al.(2003)Xu, Liu, and Gong]{xu2003document}
Xu, Wei, Liu, Xin, and Gong, Yihong.
\newblock Document clustering based on non-negative matrix factorization.
\newblock In \emph{Proceedings of the 26th Annual International ACM SIGIR
  Conference on Research and Development in Informaion Retrieval}, SIGIR '03,
  pp.\  267--273. ACM, 2003.

\bibitem[Yurochkin \& Nguyen(2016)Yurochkin and Nguyen]{yurochkin2016geometric}
Yurochkin, Mikhail and Nguyen, XuanLong.
\newblock Geometric dirichlet means algorithm for topic inference.
\newblock In \emph{Advances in Neural Information Processing Systems}, pp.\
  2505--2513, 2016.

\end{thebibliography}
\bibliographystyle{icml2017}

\end{document}